\documentclass[11pt,letterpaper]{article}

\usepackage[margin=1in]{geometry}
\usepackage{hyperref}
\usepackage{hyperref,enumitem}
\hypersetup{
	colorlinks=true,
	citecolor=blue       %
}
\setlist{noitemsep,leftmargin=\parindent,topsep=2pt}
\setlist{noitemsep,topsep=2pt}
\usepackage{natbib}
\usepackage{url}            %
\usepackage{xspace}
\usepackage{graphicx}
\usepackage{amsmath,amssymb,amsthm}
\usepackage{thm-restate}
\usepackage{booktabs} %
\usepackage{color}
\usepackage[font=small,labelfont=bf]{caption}
\usepackage{subcaption}
\usepackage{multirow}
\usepackage{authblk}

\usepackage{comment}

\usepackage[ruled]{algorithm2e} %

\SetAlFnt{\small}
\SetAlCapFnt{\small}
\SetAlCapNameFnt{\small}
\SetAlCapHSkip{0pt}
\IncMargin{-\parindent}

\newcount\Comments 
\newcommand{\kibitz}[2]{\ifnum\Comments=1{\color{#1}{#2}}\fi}
\newcommand{\zf}[1]{\kibitz{blue}{[ZF: #1]}}
\newcommand{\todo}[1]{\kibitz{blue}{[TODO: #1]}}
\newcommand{\E}{\mathbb{E}}
\DeclareMathOperator*{\argmin}{arg\,min}
\DeclareMathOperator*{\argmax}{arg\,max}
\newcommand{\1}{\mathbb{I}}

\newcommand{\eps}{\varepsilon}
\newcommand{\F}{\mathcal{F}}

\newcommand{\X}{\mathcal{X}}
\renewcommand{\L}{\mathcal{L}}
\newcommand{\rgt}{\textit{rgt}}
\newcommand{\R}{\mathbb{R}}

\newcommand{\HH}{\mathcal{H}}

\newcommand{\mbf}{\mathbf}
\newcommand{\PP}{\mathbb{P}}

\theoremstyle{plain}
\newtheorem{theorem}{Theorem}[section]
\newtheorem{corollary}{Corollary}[section]
\newtheorem{lemma}[theorem]{Lemma}
\newtheorem{assumption}[theorem]{Assumption}
\newtheorem{definition}[theorem]{Definition}
\newtheorem{proposition}[theorem]{Proposition}

\begin{document}

\title{Learning to Bid in Contextual First Price Auctions}

\author{Ashwinkumar Badanidiyuru}
\author{Zhe Feng}
\author{Guru Guruganesh}

\affil{Google Research, Mountain View \authorcr \texttt{ashwinkumarbv,zhef,gurug@google.com}}

\date{November 9, 2021}

\maketitle

\begin{abstract}
In this paper, we investigate the problem about how to bid in repeated contextual first price auctions. We consider a single bidder (learner) who repeatedly bids in the first price auctions: at each time $t$, the learner observes a context $x_t\in \mathbb{R}^d$ and decides the bid based on historical information and $x_t$. 
We assume a structured linear model of the maximum bid of all the others $m_t = \alpha_0\cdot x_t + z_t$, where $\alpha_0\in \mathbb{R}^d$ is unknown to the learner and $z_t$ is randomly sampled from a noise distribution $\mathcal{F}$ with log-concave density function $f$.
We consider both \emph{binary feedback} (the learner can only observe whether she wins or not) and \emph{full information feedback} (the learner can observe $m_t$) at the end of each time $t$.
For binary feedback, when the noise distribution $\mathcal{F}$ is known, we propose a bidding algorithm, by using maximum likelihood estimation (MLE) method to achieve at most $\widetilde{O}(\sqrt{\log(d) T})$ regret. Moreover, we generalize this algorithm to the setting with binary feedback and the noise distribution is unknown but belongs to a parametrized family of distributions. For the full information feedback with \emph{unknown} noise distribution, we provide an algorithm that achieves regret at most $\widetilde{O}(\sqrt{dT})$. Our approach combines an estimator for  log-concave density functions and then MLE method to learn the noise distribution $\mathcal{F}$ and linear weight $\alpha_0$ simultaneously. We also provide a lower bound result such that any bidding policy in a broad class must achieve regret at least $\Omega(\sqrt{T})$, even when the learner receives the full information feedback and $\mathcal{F}$ is known.
\end{abstract}

\section{Introduction}\label{sec:intro}
Recently, first price auctions have become the predominant auction mechanism on the major display advertising platforms, by replacing second price auctions~\citep{digidayFirstprice,BiglerFP}. First price auctions have grown in favor because they are more transparent and credible~\citep{AL20}, in the sense that there is no uncertainty in the final price upon winning~\citep{digidayTransparency}. Compared with the second price auctions, first price auctions are no longer truthful, i.e. reporting the true value is not the optimal strategy for each advertiser. In light of this, advertisers face new challenges in practice: how should the advertiser bid in a first price auction when it is hard to know the others' bidding strategies?

In real display ads system, a huge number of online ads are sold repeatedly via auctions. If advertisers participate in auctions very frequently to compete for placing their ads, it is very important for them to optimize their bidding strategies in repeated auctions to maximize their long term rewards. In addition, advertisers may receive some contextual information of the queries before submitting bids including information of the publisher and the user. Given this context, the advertisers can estimate their value of this query and decide their bids to compete for the ad slots. In this work, we formulate the above problem as a standard contextual online learning problem. A single advertiser (learner) repeatedly bids in contextual first price auctions and she observes a context $x_t \in \R^d$ before submitting bid at each time $t$. Then the learner submits a bid $b_t$ based on context $x_t$ and the seller use first price auctions to determine the winner and charge them their own bid.

In first price auctions, it is not enough for advertisers to bid optimally when they only know their own value, and it is necessary for them to understand the distribution of their competitors' bids~\citep{krishna-book}. In this work, we assume a structured linear model of the maximum bid of the other bidders (other than this learner) $m_t = \alpha_0 \cdot x_t + z_t$ for some unknown $\alpha_0 \in \R^d$, where $z_t \sim \F$ and the density function $f$ of $\F$ is log-concave. This assumption provides a simple model for the maximum bid of the other competitors $m_t$ and $x_t$ %and is commonly used in the economic literature~\citep{BB05} 
especially in the absence of any additional characterization. For the learner, the learning task is to simultaneously learn $\alpha_0$ and noise distribution $\F$ (if it is unknown). In repeated first price auctions, the learner can receive some information feedback at the end of each time. In this work, we  provide no-regret learning algorithms for the learner in two different information models: (1) the partial information feedback, \emph{binary feedback}, where the learner can only observe whether she wins or not; (2) the \emph{full information feedback}, where the learner can observe $m_t$ after bidding at each time $t$. 

\paragraph{Main Contributions.}
First, we characterize the optimal clairvoyant bidding strategy in contextual first price auctions, when we know the noise distribution $\F$ and parameter $\alpha_0$, in Section~\ref{sec:opt-clairvoyant-bidding}. Our characterization utilizes the log-concavity of the density function $f$ of distribution $\F$. This optimal clairvoyant bidding strategy is also used as the benchmark strategy in regret definition.

For \emph{binary feedback}, we first assume $\F$ is (fully) known and we propose a no-regret learning algorithm that achieves at most  $\widetilde{O}(\sqrt{\log(d)T})$ regret. Our algorithm is episode-based -- at each episode $s$ we use the estimated parameter $\hat{\alpha}_{s-1}$ from previous episode $s-1$ to decide the learner's current bids and update estimated parameter $\hat{\alpha}_s$ by only using the data from episode $s$ at the end. This episodic algorithm is inspired by~\citet{CCM15} and is widely used in online learning literature and has a number of advantages, e.g., it requires less computation to update parameters of the model and it can be implemented offline at the end of each episode. We utilize the maximum likelihood estimation (MLE) method to estimate $\hat{\alpha}_s$ at each episode $s$. Moreover, we extend our algorithm to the setting that $\F$ is only partially known, i.e. $\F$ is parameterized by a known based distribution $\F_0$ (e.g. standard normal distribution) and an unknown variance parameter $\sigma^2$. The regret of our algorithm for this setting is still bounded by $\widetilde{O}(\sqrt{\log(d)T})$ under some reasonable technical assumptions.

For \emph{full information feedback}, we consider the setting that $\F$ is unknown but $f$ is still log-concave. We provide an episode-based algorithm that can simultaneously  learn the noise distribution $\F$ (approximately) and the parameter $\alpha_0$. We propose a novel approach by combining the log-concave density estimator proposed in~\citep{DR09} and MLE method to learn $\F$ and $\alpha_0$ simultaneously in each episode. With reasonable assumptions (normally assumed in linear regression), our algorithm achieves at most $\tilde{O}(\sqrt{dT})$ regret and it leaves an open question that whether we can improve the algorithm with better dependence of $d$. 

Our final result shows the lower bound of regret for the full information feedback \emph{even} with known noise distribution. We consider a broad class of bidding policies and prove any algorithm in this class must incur $\Omega(\sqrt{T})$ regret for a instance. 

Despite the simple structure (greedy episodic structure), our algorithms require novel ideas and non-trivial technical contributions. In the full feedback model, we propose a new approach to combine an estimator for the log-concave functions with the MLE technique. To prove the regret bound of the algorithm, we provide a new uniform convergence bound for the log-concave noise distribution (see Theorem~\ref{thm:cdf-convergence}). This has not been known and requires a delicate balance of the parameters. In the binary feedback model with partially known noise distribution, our algorithm achieves regret bounded by $\widetilde{O}(\sqrt{\log(d)T})$. Our result improves the regret bound $\widetilde{O}(d\sqrt{T})$ proposed by~\citet{JN19} for a similar setting, by slightly strengthening the assumption of covariance matrix $\Sigma=\E[x_t x_t^T]$ (see Assumption~\ref{assump:eigenvalue-ub}). %Moreover, our analysis corrects an inaccurate regret bound result without bounded eigenvalue assumptions of the matrix $\Sigma$ in~\citep{JN19}, see Appendix~\ref{app:omitted-discussion} for more details.

\paragraph{Related Work.}
First price auctions have recieved a lot of attention in mechanism design and machine learning communities recently. For instance, \citet{WSZ20} characterizes the Bayesian Nash Equilibrium for first price auctions with discrete value and continuous bid, \citet{BalseiroEtAl21} study the equilibrium bidding strategies of contextual first price auctions with budgets, and \citet{FengEtAl21} propose a gradient-based approach to adaptively update and optimize reserve prices for first price auctions in an online manner.       

Our work is closely related with the papers in the \emph{Learning to bid} literature. The work in \citep{Balseiro19} first considers the problem of learning to bid in first price auctions by treating the value as a context. Subsequently, \citet{HZW20, HZFOW20} extended the above \emph{learning to bid} model to other settings with different feedback models and different generative models for competitors' bids \footnote{In~\citep{HZW20}, they consider a setting where $m_t$ is generated stochastically and the learner can only observe $m_t$ when she loses the auction (censored feedback). In~\citep{HZFOW20}, they consider the full information feedback model and $m_t$ comes from an adversary.}.
The main difference between our model and the above papers is that 
there is a public context (feature) $x_t\in \R^d$ observed before bidding at each time $t$ and the learner needs to decide her bid based on the value and context $x_t$. Our model allows more flexibility of the correlation between valuation and the competing bids (through context $x_t$), compared with~\citep{Balseiro19, HZW20, HZFOW20}. This is more realistic in practice, since the learner can observe some contexts before submitting the bid and she knows this will affect the competing bid as well. Loosely related works \citet{Weed16,Feng18}, consider the problem that the learner can only observe the value until she wins the auction. 

Last but not least, our work is also related with papers in the \emph{contextual pricing} field, e.g.~\citep{Mao18,PS18,JN19,Golrezaei2019,CLP20}. Especially, \citet{JN19} also assume log-concavity of the noise distribution for the valuation function in the contextual pricing problem. For the binary feedback model, our approaches generalize the methodology for contextual pricing in~\citep{JN19} to the bidding algorithms in the repeated contextual first price auctions. \citet{JN19} focus on the high-dimension setting that $\alpha_0$ is sparse and the feature dimension $d$ is larger than $T$. The algorithm proposed in~\citep{JN19} utilizes MLE with $L_1$ regularizer and it can achieve $O\left(s_0 \log d\log T\right)$ regret bound for the binary feedback model with known noise distribution and the bounded eigenvalue assumption of matrix $\Sigma = \E[x_t x_t^T]$, where $s_0$ is the sparsity parameter for $\alpha_0$, s.t. $\Vert \alpha_0\Vert_0 \leq s_0$.  In this paper, we don't assume the sparsity of context $x_t$ and we also consider the setting with full information feedback and unknown noise distribution. 
In addition,~\citet{Golrezaei2019}, propose a different way to learn noise distribution and linear weight simultaneously for contextual pricing by using \emph{ordinal least square} (OLS) method, however, this approach requires the noise distribution bounded or sub-Gaussian and cannot be applied in our full information feedback model. To address this difficulty, we propose a novel approach by combining the non-parametric log-concave density estimator and MLE method to design our bidding algorithm for the full information feedback model with unknown noise distribution.

\section{Preliminaries}\label{sec:prelim}

\subsection{Repeated Contextual First Price Auctions}
We consider the problem of online learning in repeated contextual first price auctions. There is a single seller who repeatedly sell items (e.g. ad slots in publishers) to multiple bidders (e.g. advertisers) through first price auctions. Throughout this paper, we focus on a single bidder in a large population of bidders during a time horizon $T$. In the rest of paper, we call this single bidder \emph{the learner}, who aims to maximize cumulative utility during time horizon $T$.

At each time $t$, the learner receives a public context $x_t \in \X \subseteq\R^d, \Vert x_t\Vert_2 \leq 1$ ($x_t$ is also revealed to the other bidders and the seller), and $x_t$ is i.i.d randomly generated from a prior unknown distribution $\mathcal{D}$. Based on the learner's historical information up to time $t-1$ and the realization of context $x_t$, the learner submits a bid $b_t \in [0, 1]$. Let the maximum bid of all other bidders is $m_t$ at time $t$. %
We assume there exists a \emph{known} valuation function $\beta_0: x\in \X \rightarrow [0, 1]$, which outputs the value of the learner given the input context. In other words, at each time $t$, given the realized context $x_t$, the learner can get the value $v_t = \beta_0(x_t)$. \footnote{The learner can train her own model to predict the value of each query, given the information of the query.}  In this paper, for simplicity, we assume the linear model of the context $x_t$ and maximum bid of all other bidders $m_t$, i.e. there exists an unknown parameter $\alpha_0\in \R^d$ and $\Vert \alpha_0\Vert_1 \leq W$ s.t.
\begin{eqnarray}\label{eq:linear-model-competitor}
m_t = \langle \alpha_0, x_t \rangle + z_t,
\end{eqnarray}
where $z_t$ i.i.d sampled from an unknown mean zero distribution $\F$ and $W$ is a publicly known parameter.

For notation simplicity, we allow $m_t$ be negative, which will not affect our regret results.
We call $\F$ the noise distribution and also use $f$ and $F$ to represent the probability density function (PDF) and cumulative distribution function (CDF) of the noise, i.e., $f(v) = F'(x), \forall x\in \R$. For notation simplicity, we denote $\varphi(x) = x + \frac{F(x)}{f(x)}$. 
Let $u(b, x)$ be the expected utility of the learner with bid $b$, given a context $x$, s.t.
\begin{eqnarray}\label{eq:utility-fun}
u(b, x) = \E_{z\sim \F}\left[(\beta_0(x) - b)\cdot \1\{b \geq \langle \alpha_0, x\rangle + z\}\right] =  (\beta_0(x) - b)F(b - \langle \alpha_0, x\rangle)
\end{eqnarray}
It is easy to see $u(b, x) \in [0, 1]$, for any $b, x$ and $b\leq \beta_0(x)$.
For notation simplicity, we denote $u_t(b) := u(b, x_t)$ as the utility of the learner at time $t$ with context $x_t$.

\paragraph{Feedback Models.}
In the repeated contextual first price auctions, the learner can receive different feedback at the end of each time depending on the information released from the seller. In this paper, we mainly investigate two different feedback models,
\begin{enumerate}
\item Binary feedback: the learner only observes the indicator $\delta_t = \1\{b_t \geq m_t\}$.
\item Full information feedback: the learner observes the maximum bid of all other bidders $m_t$.
\end{enumerate}

\paragraph{Regret.}
Let $\pi^*(x)$ be the optimal clairvoyant bidding strategy, suppose the learner knows $\alpha_0, \beta_0$ and distribution $F$, i.e. $\pi^*(x) = \argmax_b u(b, x)$. The target of the learner is to design a bidding strategy to decide the bid $b_t$ at each time $t$, defined in the following,
\begin{definition}[Regret]
The regret of the learner during time horizon $T$ can be defined as,
\begin{eqnarray}\label{eq:expected-utility}
R(T) = %
\E\left[\sum_{t=1}^T u_t(\pi^*(x_t))\right] - \E\left[\sum_{t=1}^T u_t(b_t)\right]
\end{eqnarray}
Here $b_t$ depends on the past history (the realization of $x_\tau, \delta_\tau$, etc. $\tau\leq t-1$), thus $(\beta_0(x_t) - b_t)F(b_t - \langle \alpha_0, x_t\rangle)$ is a random variable.
\end{definition}

\subsection{Technical Assumptions}
In addition to the linear model assumption for $m_t$, we make several assumptions of noise distribution $\F$ for the theoretical purpose. 

\begin{assumption}\label{assump:log-concave}
The density function $f$ is differentiable and log-concave. %
\end{assumption}
Log-concavity is a widely-used assumption in the economic literature \citep{BB05}. Note that if the density function $f$ is log concave, then the cumulative distribution function $F$ and the reliability function $1-F$ are both log-concave~\citep{An96}. Most common distributions such as normal, uniform, Laplace, exponential and logistic distributions satisfy the above assumption. 

In addition, we provide the following assumption of the density function $f$, 
\begin{assumption}\label{assump:bounded-noise}
There exist positive constants $B_1, B_2, B_3$, s.t.,
for any $x \in [-W, 1+W]$, $B_1 \leq f(x) \leq B_2$ and $f'(x) \leq B_3$.
\end{assumption}
Indeed, the above assumption holds for any distribution with differentiable density function $f$ and $f(x) > 0, \forall x\in [-W, 1+W]$. This implies $F$ is $B_2$-Lipschitz on $[-W, 1+W]$. The constants $B_1, B_2, B_3$ may depend on $W$. Moreover, it is straightforward to prove there exists positive constants $h_W, \ell_W$ (depending on $W$) such that,
\begin{eqnarray}\label{eq:first-order-ub}
\max\{|\log'F(x)|, |\log'(1- F(x))|\} \leq h_W, \forall x \in [-W, 1+W]
\end{eqnarray}
\begin{eqnarray}\label{eq:second-order-lb}
\min\{-\log''F(x), -\log''(1- F(x))\} \geq \ell_W, \forall x \in [-W, 1+W]
\end{eqnarray}
Eq.~(\ref{eq:first-order-ub}) holds trivially for any bounded interval $[-W, 1+W]$ and Eq.~(\ref{eq:second-order-lb}) holds because $F$ and $1-F$ are both log-concave.

\subsection{Optimal Clairvoyant Bidding Policy}\label{sec:opt-clairvoyant-bidding}
In this part, we consider the optimal clairvoyant bidding strategy if the learner knows $\alpha_0$ and noise distribution $\F$. Consider the utility of the learner at time $t$, 
\begin{eqnarray*}
u_t(b) = (\beta_0(x_t) - b)F(b - \langle \alpha_0, x_t\rangle)
\end{eqnarray*} 

Let $b_t^*$ ($b_t^* = \pi^*(x_t)$) be the optimal bid at time $t$, given context $x_t$. Suppose $b_t^* \geq 0$,
by the first-order condition, we have
\begin{eqnarray}\label{eq:opt-clairvoyant-bid-equation}
\beta_0(x_t) - b_t^* = \frac{F(b_t^* - \alpha_0 \cdot x_t)}{f(b_t^* - \alpha_0 \cdot x_t)}
\end{eqnarray}

Therefore, we have $b_t^* - \alpha_0 \cdot x_t + \frac{F(b_t^* - \alpha_0 \cdot x_t)}{f(b_t^* - \alpha_0 \cdot x_t)} = \beta_0(x_t) - \alpha_0\cdot x_t$. By the definition of function $\varphi$, we have 
\begin{eqnarray*}
\varphi(b_t^* - \alpha_0 \cdot x_t) = \beta_0(x_t) - \alpha_0\cdot x_t, \mbox{ when } b_t^* \geq 0.
\end{eqnarray*}

By the definition of function $\varphi$, we have the following proposition.
\begin{proposition}
$\varphi(\cdot)$ is a strictly increasing function and $0 < (\varphi^{-1}(x))' < 1$ for all $x\in \R$.
\end{proposition}
\begin{proof}
Let $\lambda(x) = \frac{f(x)}{F(x)}$, then $\varphi(x) = x + \frac{1}{\lambda(x)}$. By Assumption~\ref{assump:log-concave}, $\lambda(x) = \log'F(x)$ is decreasing. Then $1/\lambda(x)$ is increasing, which implies $\varphi(x)$ strictly increasing. In addition, $\varphi'(x)  > 1, \forall x$, since $1/\lambda(x)$ is increasing. Then for any $x$, $(\varphi^{-1}(x))' = \frac{1}{\varphi'(\varphi^{-1}(x))} < 1$.
\end{proof}
The above proposition implies the optimal bid $b_t^*$, given context $x_t$ can be represented as,
\begin{eqnarray}\label{eq:opt-bid}
b_t^* = \max\{0, \alpha_0\cdot x_t +  \varphi^{-1}(\beta_0(x_t) - \alpha_0 \cdot x_t)\}
\end{eqnarray}

It is straightforward to verify $b_t^* \leq \beta_0(x_t)$, since $x + \varphi^{-1}(v - x) \leq v, \forall v, x\in \R$.
Given the characterization of the above optimal clairvoyant bidding strategy, we can rewrite the regret:
\begin{eqnarray}\label{eq:worse-case-regret}
R(T) = %
\E\left[\sum_{t=1}^T u_t(b_t^*)\right] - \E\left[\sum_{t=1}^T u_t(b_t)\right],
\end{eqnarray}
\section{Binary Feedback Model}\label{sec:binary-feedback}
In this section, we consider the least information feedback model --- binary feedback that the learner can only observe whether she wins or not at the end of each time.

\subsection{Binary Feedback with Known Noise Distribution}\label{sec:binary-known-noise}
In this section, we assume the learner knows the noise distribution $\F$, i.e., $f$ and $F$ are known. In this case, the learner only needs to learn $\alpha_0$. %

\paragraph{Algorithm.}
Our bidding algorithm runs in an episode manner, similarly to \citet{CCM15, JN19}. During a time horizon $T$, the bidding algorithm is divided into $S$ episodes, where each episode contains $T_s$ time steps. Denote $\Gamma_s$ be the time steps in stage $s$, s.t. $|\Gamma_s| = T_s$. For any time step $t$ in the first episode, we simply set $b_t = 1$. For any time step $t$ in episode $s (s\geq 2)$, i.e., $t\in \Gamma_s$, we set the bid 
\begin{eqnarray}\label{eq:binary-known-noise-bid}
b_t = \max\{0, \hat{\alpha}_{s-1} \cdot x_t + \varphi^{-1}(\beta_0(x_t) - \hat{\alpha}_{s-1} \cdot x_t)\}
\end{eqnarray}
for the learner, where $\hat{\alpha}_{s-1}$ is the estimation of $\alpha_0$ based on the observations $\{x_t, \delta_t, b_t\}, t \in \Gamma_{s-1}$ in the previous episode $s-1$. Indeed, we replace $\alpha_0$ by $\hat{\alpha}_{s-1}$ in the optimal clairvoyant bidding policy shown in Eq.~(\ref{eq:opt-bid}) to set the bid $b_t$ at time $t, \forall t > T_1$. If the estimator $\hat{\alpha}_{s-1}$ is close to $\alpha_0$ based on the observations in the episode $s-1$, the expected utility $u_t(b_t)$ will be close to the optimal expected utility $u_t(b^*_t)$ (see Lemma~\ref{lem:binary-known-noise-utility-difference}).

Given the above definition, we show the pseudo code of our bidding algorithm for this setting in Algorithm~\ref{alg:binary-known-noise-mle}.
In each episode $s$, we estimate $\alpha_0$ by using \emph{maximum likelihood estimation} (MLE) method. Specifically, we notice at each time $t$,
$$
\delta_t = \left\{
\begin{array}{cc}
1 & \mbox{ w.p. } F(b_t - \alpha_0\cdot x_t)\\
0 & \mbox{ w.p. } 1 - F(b_t - \alpha_0\cdot x_t)
\end{array}
\right.
$$

Therefore, we denote $\mathcal{L}_s(\alpha)$ be the negated log-likelihood function for $\alpha$ in the episode $s$,
\begin{eqnarray}\label{eq:binary-known-noise-likelihood}
\mathcal{L}_s(\alpha) = -\frac{1}{T_s} \sum_{t\in \Gamma_s} \left[\delta_t\cdot \log F(b_t - \alpha\cdot x_t) + (1-\delta_t)\cdot \log(1-F(b_t - \alpha\cdot x_t))\right],
\end{eqnarray}
where $\delta_t = \1\{b_t \geq m_t\}$. Indeed, based on our log-concavity assumption on $F$ and $1-F$, the negated log-likelihood $\L_s(\alpha)$ is convex for any $s=1, 2,\cdots, S$. Therefore, we can run standard gradient descent algorithm to minimize loss function $\L_s(\alpha)$.

\begin{algorithm}[t]
\SetAlgoNoLine
\KwIn{Parameters $W, T$, function $\varphi^{-1}(\cdot)$}
\For{$t\in \Gamma_1$}{
The learner observes $x_t$ and submits a bid $b_t = 1$. The learner observes $\delta_t$.
}

Estimate $\alpha_0$ by using $\hat{\alpha}_1$, which is computed by $\hat{\alpha}_1 = \argmin_{\Vert\alpha\Vert_1 \leq W} \mathcal{L}_1(\alpha)$.

\For{episode $s=2,3,\cdots, S$}{
\For{$t\in \Gamma_s$}{
The learner observes $x_t$ and submits $b_t$, where $b_t$ is computed in the following way,
\begin{eqnarray*}
b_t = \max\{0, \hat{\alpha}_{s-1}\cdot x_t + \varphi^{-1}(\beta_0(x_t) - \hat{\alpha}_{s-1}\cdot x_t)\}
\end{eqnarray*}
The learner observes $\delta_t$.
}

Update the estimator for $\alpha_0$ in episode $s$ by $\hat{\alpha}_s = \argmin_{\Vert\alpha\Vert_1 \leq W} \mathcal{L}_s(\alpha)$.
}
\caption{Bidding algorithm in the binary feedback model with known noise distribution}\label{alg:binary-known-noise-mle}
\vspace{-2pt}
\end{algorithm}

\paragraph{Regret Analysis.}
We show the regret bound for the setting considered in this subsection as below, and the full proof is deferred to Appendix~\ref{app:binary-feedback}.

\begin{restatable}{theorem}{binarynoiseregret}
\label{thm:binary-known-noise-regret}
Suppose Assumption~\ref{assump:bounded-noise} and Assumption~\ref{assump:log-concave} hold, setting $T_s = T^{1-2^{-s}}, s = 1,2,\cdots$, then with probability at least $1-\delta$, the regret achieved in the binary feedback model with known noise distribution is at most
$R(T) \leq \widetilde{O}\left(\sqrt{\log(d/\delta) T}\right)$,
where $\widetilde{O}$ omits $\log\log T$ terms.
\end{restatable}
\begin{proof}[Proof Sketch]
Our proof follows the same spirit as in Theorem 4 in~\citep{JN19}. First, we bound $\vert u_t(b_t^*) - u_t(b_t)\vert$ at each time $t\in \Gamma_s, (s\geq 2)$ by $\Theta\left(\vert x_t \cdot (\alpha_0 - \hat{\alpha}_{s-1}) \vert^2\right)$ in Lemma~\ref{lem:binary-known-noise-utility-difference} (Appendix~\ref{app:binary-feedback}). Then, we can bound the regret achieved in each episode $s$ by $\Theta\left(\sum_{t\in \Gamma_s}\langle \alpha_0 - \hat{\alpha}_{s-1}, \Sigma (\alpha_0 - \hat{\alpha}_{s-1})\rangle\right)$, where $\Sigma = \E[x_t x_t^T]$. Next we decompose $\langle \alpha_0 - \hat{\alpha}_{s-1}, \Sigma (\alpha_0 - \hat{\alpha}_{s-1})\rangle$ by $ \frac{1}{T_{s-1}} \sum_{t\in \Gamma_{s-1}}\langle \alpha_0 - \hat{\alpha}_{s-1}, x_t x^T_t (\alpha_0 - \hat{\alpha}_{s-1})\rangle + \langle \alpha_0 - \hat{\alpha}_{s-1},  E(\alpha_0 - \hat{\alpha}_{s-1})\rangle$, where the first term is bounded by Lemma~\ref{lem:second-order-concentration} (Appendix~\ref{app:binary-feedback}) and $\Vert E\Vert_\infty$ is bounded by $O\left(\sqrt{\frac{\log(d/\delta)}{T_{s-1}}}\right)$ w.h.p. by matrix Hoeffding's inequality. Aggregating over all episodes, we show the total regret bound. %
\end{proof}

\noindent{\textbf{Remark.}} Theorem~\ref{thm:binary-known-noise-regret} doesn't rely on the assumption of the bounded eigenvalue of matrix $\Sigma=\E[x_t x_t^T]$ and the sparsity assumption of parameter $s_0$. With these two assumptions, ~\citet{JN19} show it can achieve $O(s_0 \log d \log T)$ regret bound for this binary feedback model with known noise distribution.
\subsection{Extension to Partially-Known Noise Distribution}\label{sec:binary-partial-known-noise}

In this section, we extend to the case that the noise distribution $\F$ is parameterized by a zero-mean base noise distribution $\F_0$ and a variance $\sigma^2 (\sigma > 0)$, where $\F_{0}$ is known (e.g. $\mathcal{N}(0, 1)$) but $\sigma$ is unknown. We denote $\rho_0 = \frac{1}{\sigma}$. Without loss of generality, we assume $\vert\rho_0 \vert\leq W$. In this case, the learner needs to simultaneously learn $\alpha_0$ and $\sigma$. %
We denote $f_0$ and $F_0$ be the density function and cumulative function of distribution $\F_0$, which are known to the learner. Let $\varphi_0(x) = x + \frac{F_0(x)}{f_0(x)}$.

\paragraph{Modified Algorithm.}
The algorithm follows the same fashion of Algorithm~\ref{alg:binary-known-noise-mle} and we show the pseudo code in Algorithm~\ref{alg:binary-partial-known-noise-mle} in Appendix~\ref{app:omitted-algorithms}. The main difference in this algorithm is how to estimate $\hat{\alpha}_0$ and $\rho_0$ simultaneously. By the definition of $\delta_t = \1\{b_t\geq m_t\}$, we observe $\delta_t = 1$ with probability $F_0(\rho_0(b_t - \alpha_0\cdot x_t))$ and $\delta_t = 0$ with probability $1 - F_0(\rho_0(b_t - \alpha_0\cdot x_t))$.
To simplify presentation, we re-parametrize $\alpha_0, \rho_0$ by denoting $\mu_0 = \alpha_0 \rho_0$ and write the negated log-likelihood function in each episode $s$ as follows,
\begin{eqnarray}\label{eq:mle-partial-known-noise}
\mathcal{L}_s(\mu, \rho) &=& -\frac{1}{T_s} \sum_{t\in \Gamma_s} \left(\delta_t\cdot \log F_0(\rho b_t - \mu\cdot x_t) + (1-\delta_t)\cdot \log(1-F_0(\rho b_t - \mu\cdot x_t))\right)
\end{eqnarray}

In Algorithm~\ref{alg:binary-partial-known-noise-mle}, we always update the estimator $(\hat{\mu}_s, \hat{\rho}_s)$ of $(\mu_0, \rho_0)$ in a valid set $\Lambda$ by minimizing the loss function (Eq.~\ref{eq:mle-partial-known-noise}) at the end of each episode $s$.
\begin{eqnarray}\label{eq:valid-set-partial-known-noise}
\Lambda = \{(\mu, \rho)|\mu\in \R^d, \rho > 0, \Vert\mu/\rho\Vert_1\leq W, \vert\rho\vert\leq W\}
\end{eqnarray}
Then, for any $(\mu, \rho)\in \Lambda$ and $b_t\in [0, 1]$, $\rho b_t - \mu\cdot x_t \in [-W^2, W^2+W]$. For each time step $t$ in the first episode, we set the bid $b_t = 1$. For any time $t$ in episode $s$, we set the bid
\begin{eqnarray}
b_t = \max\left\{\Delta, \frac{1}{\hat{\rho}_{s-1}} \left(\hat{\mu}_{s-1}\cdot x_t + \varphi^{-1}_0(\hat{\rho}_{s-1}\beta_0(x_t) - \hat{\mu}_{s-1}\cdot x_t)\right)\right\}.
\end{eqnarray}
As the astute readers may notice, we only consider the case that the bids are larger than a small positive constant $\Delta$ (i.e. one cent), in this setting. For theoretical purpose, the assumption that $b_t \geq \Delta$ guarantees the strong convexity of $\L_s(\mu, \rho)$ w.r.t $\rho$ so that we can bound $(\hat{\rho}_s - \rho_0)^2$ in each episode. For practical perspective, this assumption holds trivially since the display ads platform usually requires a minimum amount of bid, e.g. one cent, to compete for ad slots. We replace $(\mu_0, \rho_0)$ by the estimator $(\hat{\mu}_{s-1}, \hat{\rho}_s)$ in the optimal bidding policy\footnote{Since $F_0$ and $1- F_0$ are both log-concave, the optimal clairvoyant bidding policy (without truncation to $\Delta$) is $\frac{1}{\rho_0}\left(\mu_0\cdot x_t + \varphi^{-1}_0(\rho_0\beta_0(x_t) - \mu_0\cdot x_t\right)$ following the same argument in Section~\ref{sec:opt-clairvoyant-bidding}.}.

For theoretical purpose, we need the following assumption on the product context,
\begin{assumption}\label{assump:eigenvalue-ub}
The maximum eigenvalue of matrix $\Sigma = \E[x_t x_t^T]$ is bounded by a constant $\lambda_{1} > 0$. The minimum eigenvalue of $\Sigma$ is bounded from below by a constant $\lambda_{2} > 0$. In addition, we assume $\Sigma - \E[x_t]\E[x_t]^T \succ \lambda_3 I$ for a constant $\lambda_3 > 0$.
\end{assumption}
The assumption on the bounded eigenvalues of matrix $\Sigma$ is commonly proposed in the convergence analysis of the linear models. It is well-known $\Sigma - \E[x_t]\E[x_t]^T$ is positive semi-definite and we strengthen it to be strictly positive definite here.
Indeed, the above assumption holds for many common probability distributions of context $x_t$, such as uniform, truncated normal and in general truncated version of many more distributions.

\paragraph{Regret Analysis.}
To begin with, we state the benchmark in the regret analysis considered in this section. %
To be consistent with our bidding space, we consider a slightly weaker but practical benchmark, i.e. the bids are all truncated above $\Delta$. 
Therefore, for any realized context $x_t$,  the \emph{optimal bidding policy}  (benchmark) is
\begin{eqnarray}\label{eq:opt-bid-binary-partial-known-noise}
b_t^* = \max\left\{\Delta, \frac{1}{\rho_0}\left(\mu_0\cdot x_t + \varphi^{-1}_0(\rho_0\beta_0(x_t) - \mu_0\cdot x_t\right)\right\}
\end{eqnarray}

Comparing with this benchmark, we state our main theorem of the regret bound in this section, and the proof is deferred to Appendix~\ref{app:binary-partial-known-noise}.

\begin{restatable}{theorem}{BinaryPartialNoise}
\label{thm:regret-binary-partial-known-noise}
Suppose Assumptions~\ref{assump:log-concave},~\ref{assump:bounded-noise} and~\ref{assump:eigenvalue-ub} hold, setting $T_s = T^{1-2^{-s}}, s = 1,2,\cdots$, then with probability at least $1-\delta$, the regret (w.r.t the benchmark defined in Eq.~(\ref{eq:opt-bid-binary-partial-known-noise})) achieved in the binary feedback model with partially known noise distribution is bounded by
$R(T) \leq \widetilde{O}\left(\sqrt{\log(d/\delta) T}\right)$,
where $\widetilde{O}$ ignores $\log T$ and $\log\log T$ terms.
\end{restatable}

\noindent\textbf{Remark.}
\citet{JN19} study the contextual pricing problem in a very similar setting, i.e. the noise distribution of valuation belongs to a known (parameterized) class with unknown parameters. Our result improves the regret bound $\widetilde{O}(d\sqrt{T})$ proposed in~\citep{JN19} by using a slightly stronger assumption of $\Sigma$ (Assumption~\ref{assump:eigenvalue-ub})\footnote{Indeed, their algorithm achieves $O(s_0 \sqrt{T})$ regret bound, where $s_0$ is the sparsity parameter of $\alpha_0$. In this paper, we have no sparsity assumption and $s_0$ can be equal to $d$.}.
% (2) our analysis fixes an incorrect result that~\citet{JN19} claim their algorithm can achieve $\widetilde{O}(\sqrt{\log(d)T})$ regret bound \emph{without} Assumption~\ref{assump:eigenvalue-ub} in this setting (see more detailed discussion in Appendix~\ref{app:omitted-discussion}).
%Their algorithm can achieve $\widetilde{O}(\sqrt{\log(d)T})$ regret upper bound without Assumption~\ref{assump:eigenvalue-ub}. However, we believe there is a mistake in their proof ( Lemma 20 in page 44, see more detailed discussion in Appendix~\ref{app:omitted-discussion}). 
In our proof, we show the loss function $\L_s(\mu, \rho)$ are strongly convex with high probability, in Lemma~\ref{lem:binary-partial-known-noise-strong-convexity}. The proof for this Lemma utilizes Schur Complements and advanced matrix inequalities. %We believe Assumption~\ref{assump:eigenvalue-ub} is necessary for us to bound $\Vert (\hat{\mu}_s, \hat{\rho}_s) - (\mu_0, \rho_0)\Vert_2^2$.

\section{Full Information Feedback Model}\label{sec:full-unknown-noise}

In this section, we consider the full information feedback model with unknown noise distribution, i.e., the learner has no information of noise distribution, however she can always observe the highest bid of all other bidders $m_t$. Without knowledge of noise distribution $\F$, the learner cannot directly use naive MLE method to estimate $\alpha_0$ used in Section~\ref{sec:binary-known-noise}. %

Following the same spirit as in Section~\ref{sec:binary-known-noise}, we still build our algorithm be \emph{episode-based}, i.e. at each episode $s$, we use the estimated noise distribution $\hat{F}_{s-1}$ and parameter $\hat{\alpha}_{s-1}$ from the $(s-1)$th episode to determine the learner's bid and only update these estimators at the end of episode $s$ by the using the data observed in episode $s$. The main difficulty is how to update the estimators of $\F$ and $\alpha_0$ in each episode. To handle this challenge, we propose a new approach, combining the non-parametric log-concave density estimator and MLE method, to learn $\alpha_0$ and $\F$ simultaneously.

\paragraph{Non-parametric estimation of $f$.}
We first introduce the non-parametric estimation of density function $f$, given any linear weight estimator $\alpha$. This non-parametric estimator of $f$ is from~\citep{DR09} and we generalize it here to incorporate with different estimation of $\alpha_0$. In each episode $s$, given realized $x_t, m_t, t\in \Gamma_s$ and any linear weight estimator $\alpha$.
\begin{eqnarray}\label{eq:non-parametric-estimator}
\hat{f}_s(\cdot; \alpha) = \argmax_{f \mbox{ is log-concave}} \frac{1}{T_s} \sum_{t\in \Gamma_s} \log f(m_t - \alpha\cdot x_t) - \int f(z) dz
\end{eqnarray}

For notation simplicity, it is without loss of generality to re-parameterize 
$f(z) = \exp(\Psi(z))$, where $\Psi(z)$ is a concave function w.r.t $z$. Then given any linear weight $\alpha$, it is equivalent to optimize estimator $\hat{\Psi}_s(\cdot; \alpha)$ to get an estimator $\hat{f}_s(\cdot; \alpha)$ in each episode $s$, in the following,
\begin{eqnarray}\label{eq:concave-optimizer-episode}
\hat{\Psi}_s(\cdot; \alpha) = \argmax_{\Psi \mbox{ is concave}} \frac{1}{T_s} \sum_{t\in \Gamma_s} \Psi (m_t - \alpha\cdot x_t) - \int \exp\left(\Psi(z; \alpha)\right) dz
\end{eqnarray}

Denote $\hat{F}_s(z; \alpha) = \int^z \hat{f}_s(t; \alpha) dt$ be the estimated empirical distribution given linear weight estimator $\alpha$, in each episode $s$. In this work, we restrict the function class of $\hat{f}_s(\cdot; \alpha)$ for any $\alpha$ and $s$ as below,
\begin{eqnarray*}
\mathcal{P} = \left\{p: p(z) \leq B_2, \forall z\in [-W, 1+W], \int g(z)dz = 1\right\}
\end{eqnarray*}
This implies $\hat{\Phi}_s(z; \alpha) \leq \log B_2$ for any $z\in [-W, 1+W]$\footnote{We only care about $z\in [-W, 1+W]$ because we only need to estimate $F$ on $[-W, 1+W]$ to estimate expected utility function, see Eq.~(\ref{eq:utility-fun}).} and any $\alpha$ s.t. $\Vert \alpha\Vert_1 \leq W$. In addition, it is straightforward to $\hat{F}_s(\cdot; \alpha)$ is $B_2$-Lipschitz.

Let $\mathbb{F}_s$ be the empirical distribution of noise samples $\{z_t\}_{t\in \Gamma_s}$ in episode $s$, therefore we have $\mathbb{F}_s(z) = \frac{1}{T_s}\sum_{t\in \Gamma_s}\1\{z\leq z_t\}$.
\citet{DR09}, characterizes the optimizer $\hat{\Phi}_s(\cdot; \alpha)$ as well as estimator $\hat{F}_s(z; \alpha)$ when $\alpha= \alpha_0$, in the following,

\begin{lemma}[\citep{DR09}]\label{lem:empirical-cdf-knots}
The optimizer $\hat{\Phi}_s(\cdot; \alpha_0)$ exists and is unique. For any $z_t = m_t -\alpha_0\cdot x_t, t\in \Gamma_s$, $\mathbb{F}_s(z_t) - \frac{1}{T_s} \leq \hat{F}_s(z_t; \alpha_0) \leq \mathbb{F}_s(z_t)$.
\end{lemma}

Give the above characterization of $\hat{\Phi}(\cdot; \alpha_0)$ and $\hat{F}_s(\cdot; \alpha_0)$ we provide the uniform convergence bound for $\vert \hat{F}_s(z; \alpha_0) - F(z)\vert$ in the following Theorem, the proof is deferred to Appendix~\ref{app:cdf-convergence}.

\begin{restatable}{theorem}{CDFConvergence}
\label{thm:cdf-convergence}
Suppose %
$T_s \gg \log^2(1/\delta)$ for any fixed $\delta > 0$, then for all $z\in [-W, 1+W]$, $\vert\hat{F}_s(z; \alpha_0) - F(z)\vert \leq O\left(\sqrt{\frac{\log(1/\delta)}{T_s}}\right)$ holds with probability at least $1 - \delta$.
\end{restatable}

\if 0
\begin{assumption} [\textcolor{red}{Because we restrict the shape of $\hat{F}_s(\cdot; \alpha)$ and it is WLOG}]\label{asmpt:piece-wise-linear}
For all $s$,
\begin{eqnarray*}
\max\{|\log'\hat{F}_s(x; \alpha_0)|, |\log'(1- \hat{F}_s(x; \alpha_0))|\} \leq \tilde{h}, \forall x \in [-W, 1+W]
\end{eqnarray*}
\begin{eqnarray*}
\min\{-\log''\hat{F}_s(x; \alpha_0), -\log''(1- \hat{F}_s(x; \alpha_0))\} \geq \tilde{\ell}, \forall x \in [-W, 1+W]
\end{eqnarray*}
\end{assumption}
\fi

\paragraph{Algorithm.}
Similarly, we assume there are $S$ episodes in the algorithm, each episode $s$ contains $T_s$ time steps, and $\Gamma_s$ be the set of time steps in episode $s$.
Given the non-parametric estimator of $\hat{F}_s(\cdot; \alpha)$ introduced in the above, we introduce our algorithm for the full information feedback model:
\begin{itemize}[leftmargin=*]
\item For any time step $t$ in the first episode, the learner sets the bid $b_t = 1$.
\item For any time step $t$ in episode $s (s\geq 2)$, i.e. $\forall t\in \Gamma_s$, the learner sets the bid
\begin{eqnarray}
b_t = \argmax_{b\in [0, 1]} (\beta_0(x_t) - b)\hat{F}_{s-1}(b - \hat{\alpha}_{s-1}\cdot x_t; \hat{\alpha}_{s-1}),
\end{eqnarray}
\end{itemize}
where $\hat{\alpha}_{s-1}$ is the estimator of $\alpha_0$ based on the data observed in episode $s-1$ and $\hat{F}_{s-1}(\cdot; \hat{\alpha}_{s-1})$ is the estimator of noise distribution (CDF) shown in Eq.~(\ref{eq:non-parametric-estimator}). To compute $\hat{\alpha}_s$, we minimize the following MLE loss function,
\begin{eqnarray}\label{eq:full-info-log-likelihood}
\L_s(\alpha) := -\frac{1}{T_s} \sum_{t\in \Gamma_s}\left(\delta_t \cdot \log \hat{F}_s(\eps_t(\alpha); \alpha) + (1-\delta_t)\cdot \log\left(1- \hat{F}_s(\eps_t(\alpha); \alpha)\right)\right),
\end{eqnarray}
where $\eps_t(\alpha) = b_t - \alpha\cdot x_t$ and $\Vert \alpha\Vert_1 \leq W$. The pseudo-code is presented in Algorithm~\ref{alg:full-unknown-noise-mle} in Appendix~\ref{app:omitted-algorithms}. Indeed, $\L_s(\alpha)$ is convex almost everywhere, since $\hat{F}_s$ and $1-\hat{F}_s$ are both log-concave based on our construction. We can still solve this optimization problem by gradient descent approach, but we need to recompute $\hat{F}_s(\cdot; \alpha)$ to get the gradient of loss function $\L_s$ at $\alpha$ in each iteration of gradient descent. If we can compute $\hat{F}_s$ efficiently, combining with gradient descent approach, our algorithm is computationally efficient. In this paper, we focus on regret analysis and leave the computational efficiency argument as a future direction. %

\paragraph{Regret Analysis.}
We provide the regret bound for the full information feedback model in the following Theorem. The full proof is deferred to Appendix~\ref{app:full-info-regret}.

\begin{restatable}{theorem}{FullInfoRegret}[Regret Bound]
\label{thm:full-info-regret}
Suppose Assumptions~\ref{assump:log-concave},~\ref{assump:bounded-noise}, and~\ref{assump:eigenvalue-ub} hold~\footnote{In fact, we only need the assumption that the minimum eigenvalue of matrix $\Sigma$ is larger than $\lambda_2 > 0$ in this setting.} , and $T$ is sufficiently large.
Given %
$T_s = T^{1-2^{-s}}, s = 1,2,\cdots,$ then with probability at least $1-\delta$, the regret is bounded by
$R(T) \leq \widetilde{O}\left(\sqrt{d\log(d/\delta) T}\right)$,
where $\widetilde{O}$ ignores $\log T$ and $\log\log T$ terms.
\end{restatable}
\begin{proof}[Proof Sketch]
The main challenge in this proof is to bound the difference between estimator $\hat{\alpha}_s$ and $\alpha_0$, as well as the distance between $\hat{F}_s$ and $F$. First, we give a bound of $L_2$ distance between $\hat{\alpha}_s$ and $\alpha_0$ in Lemma~\ref{lem:full-info-convergence}. The proof of this Lemma strictly generalizes the idea of Theorem~\ref{thm:binary-known-noise-regret}, combining with the uniform convergence bound of $\vert \hat{F}_s(z; \alpha_0) - F(z)\vert$ in Theorem~\ref{thm:cdf-convergence}. Then, we show if $\Vert \hat{\alpha}_s -\alpha_0\Vert_2 \leq O(\sqrt{\frac{d\log(T_s)}{T_s}})$ holds, then $\vert \hat{F}_s(z; \hat{\alpha}_s) - F(z)\vert \leq O(\sqrt{\frac{d\log(T_s)}{T_s}})$ for all $z\in [-W, 1+W]$ holds with high probability in Lemma~\ref{lem:empirical-cdf-convergence}. Since Lemma~\ref{lem:full-info-convergence} implies $\Vert \hat{\alpha}_s -\alpha_0\Vert_2 \leq O(\sqrt{\frac{d\log(T_s)}{T_s}})$ holds with high probability, then we provide a uniform convergence for $\vert \hat{F}_s(z; \hat{\alpha}_s) - F(z)\vert$. %
\end{proof}

\section{Lower Bound}\label{sec:lower-bound}

In this section, we show the lower bound of regret for the full information feedback model with known noise distribution, i.e. $\F$ is known and $m_t$ is always realized at the end of each time $t$. 

As we know, if $\alpha_0$ is known, the optimal bidding strategy is
\begin{eqnarray*}
b^*_t =  \max\{0, \alpha_0\cdot x_t + \varphi^{-1}(\beta_0(x_t) - \alpha_0\cdot x_t)\}.
\end{eqnarray*}
Let $\HH_t = \{x_1, x_2,\cdots, x_t, m_1, m_2,\cdots, m_t\}$ be the history observed up to time $t$ and we consider the following set of bidding policies, $\Pi$:
\begin{equation}\label{eq:bidding-policy}
\begin{aligned}
\Pi = \{\pi: (\HH_{t-1}, x_t) &\rightarrow b_t = \max\{0, \alpha_t\cdot x_t + \varphi^{-1}(\beta_0(x_t) - \alpha_t\cdot x_t)\},\\
&\Vert\alpha_t\Vert_1 \leq W, \alpha_t \mbox{ is $\HH_{t-1}$-measurable.}\}
\end{aligned}
\end{equation}

Here $\alpha_t$ can be regarded as an (inaccurate) estimator of $\alpha_0$ and $\Pi$ captures a wide class of \emph{informational} bidding policies\footnote{Informative bidding policy means the learner can always gain some information of parameter $\alpha_0$ by varying bids. It is without loss generality that we focus on informational bids since there exists no "uninformational" bids in the setting presented in Theorem~\ref{thm:lower-bound-known-F}. See more discussion in Appendix~\ref{app:omitted-discussion}}. 
Indeed, when we restrict our attention on the bidding policies in $\Pi$, we can derive any bidding policy $\pi\in \Pi$ must incur expected $\Omega(\sqrt{T})$ in the following theorem. The proof is rather technical and we defer it to Appendix~\ref{app:lower-bound}.

\begin{theorem}\label{thm:lower-bound-known-F}
For any $T$, we assume that the market value $z_t, 1 \leq t \leq T$ are fully observed. We further assume $z_t \sim \mathcal{N}(0, \sigma^2)$, where $\sigma$ is known. Let $\Pi$ be the set of bidding polices $\pi$ defined in Eq.~(\ref{eq:bidding-policy}), then any bidding policy $\pi$ must incur expected regret $\Omega(\sqrt{T})$.
\end{theorem}

\section{Future Work}\label{sec:conclusion}
In this paper, we assume the linear model of $m_t$ w.r.t. context $x_t$ and a natural future direction is to extend  to non-linear model. We assume the context $x_t$ is randomly sampled from a fixed, prior unknown distribution. It will be interesting to design a no-regret bidding algorithm for contextual first price auctions when the context is generated from adversary. In the future, we are interested in generalizing our algorithms to other contextual untruthful (beyond first price auctions). In addition, we assume the learner can estimate the value $\beta_0(x_t)$ before submitting the bid and it would be exciting to incorporate with the setting that the learner cannot observe the value unless she wins the auctions.

\bibliographystyle{abbrv}
\bibliography{references}

\newpage
\appendix
\appendix
\begin{center}
{
\Large
\textbf{
Learning to Bid in Contextual First Price Auctions}
~\\
~\\	
Appendix
}
\end{center}

\section{Useful Technical Lemmas}

\begin{lemma}[Schur complement~\citep{SchurComplement}]\label{lem:schur-complement}
Let
$$
M = 
\begin{bmatrix}
A & B\\
B^T & C
\end{bmatrix}
$$
where $A$ positive definite (invertible) and $C$ is symmetric, then the matrix $C-B^TA^{-1}B$ is called the Schur complement of $A$. Then $M \succeq 0$ iff $C-B^T A^{-1}B \succeq 0$.
\end{lemma}

\begin{lemma}[Matrix Inverse Lemma]\label{lem:matrix-inverse}
Let $A$ be invertible, for any constant $\lambda > 0$, we have
\begin{eqnarray*}
(A+\lambda I)^{-1} = A^{-1} - A^{-1}\left(\frac{1}{\lambda} I + A^{-1}\right)^{-1}A^{-1}
\end{eqnarray*}
\end{lemma}

\section{Omitted Proofs from Section~\ref{sec:binary-feedback}}\label{app:binary-feedback}

\subsection{Proof of Theorem~\ref{thm:binary-known-noise-regret}}\label{app:binary-known-noise}

For convenience, we restate the theorem: 
\binarynoiseregret*

To prove Theorem~\ref{thm:binary-known-noise-regret}, we introduce some auxiliary lemmas. Our proof is inspired by~\citet{JN19}. First we bound the difference between optimal expected utility and the expected utility achieved by our bidding algorithm at each time $t$ in the episode $s (s\geq 2)$ by $\Theta(|x_t\cdot (\alpha_0 - \hat{\alpha}_{s-1})|^2)$. The proof involves an case analysis.

\begin{lemma}\label{lem:binary-known-noise-utility-difference}
For any $s\geq 2$ and any $t\in \Gamma_s$, let $b_t^*$ be the optimal bid given context $x_t$,
\begin{eqnarray*}
u_t(b_t^*)  - u_t(b_t) \leq 2C |x_t \cdot (\alpha_0 - \hat{\alpha}_{s-1})|^2
\end{eqnarray*}
with $C = 2B_2 + B_3$, where $B_2$ and $B_3$ are positive constants defined in Assumption~\ref{assump:bounded-noise}.
\end{lemma}

\begin{proof}
Firstly, for any $t\in \Gamma_s$, we prove $\vert u''_t(b)\vert \leq C, \forall b\in [0, 1]$, with $C = 2B_2 + B_3$.
\begin{eqnarray*}
\vert u''_t(b)\vert \leq \vert (\beta_0(x_t) - b) f'(b - \alpha_0\cdot x_t) - 2f(b - \alpha_0 \cdot x_t)\vert \leq 2B_2 + B_3
\end{eqnarray*}

Denote $\hat{b}^*_t = \alpha_0\cdot x_t +  \varphi^{-1}(\beta_0(x_t) - \alpha_0 \cdot x_t)$, $\hat{b}_t = \hat{\alpha}_{s-1} \cdot x_t + \varphi^{-1}(\beta_0(x_t) - \hat{\alpha}_{s-1} \cdot x_t)$, for any $t\in \Gamma_s$. Since $\hat{b}_t^*$ maximize $u_t(b)$, $u'_t(\hat{b}_t^*) = 0$. Then we bound $u_t(b_t^*) - u_t(b_t)$ by a case analysis,

\begin{itemize}
\item $\hat{b}^*_t \geq 0$, then $b_t^* = \hat{b}^*_t$. By second-order Taylor's theorem, we have
\begin{eqnarray*}
u_t(b_t) = u_t(b_t^*) + u'_t(b_t^*)(b_t - b_t^*) + \frac{1}{2} u''_t(\tilde{b}) (b_t - b_t^*)^2,
\end{eqnarray*}
for some $\tilde{b}$ between $b_t$ and $b_t^*$. Since $u'_t(b_t^*) = u'_t(\hat{b}_t^*) = 0$, we have
\begin{eqnarray*}
u_t(b_t^*) - u_t(b_t) \leq \frac{1}{2}\vert u''_t(\tilde{b})\vert(b_t - b_t^*)^2 \leq \frac{C}{2}(b_t - b_t^*)^2 \leq \frac{C}{2} (\hat{b}_t - \hat{b}^*_t)^2
\end{eqnarray*}

\item $\hat{b}_t^* < 0$, then $b_t^* = 0$. 
\begin{itemize}
\item When $\hat{b}_t < 0$, $b_t = 0$. Thus, $u_t(b^*_t) - u_t(b_t) = 0$.
\item When $\hat{b}_t \geq 0$, $b_t = \hat{b}_t \geq 0$. Then by second-order Taylor's theorem, we have
\begin{eqnarray*}
	u_t(b_t) = u_t(\hat{b}_t^*) + u'_t(\hat{b}_t^*)(b_t - b_t^*) + \frac{1}{2} u''_t(\tilde{b}) (b_t - \hat{b}_t^*)^2,
\end{eqnarray*}
for some $\tilde{b}$ between $b_t$ and $\hat{b}_t^*$. Since $u_t(b^*_t) \leq u_t(\hat{b}_t^*), u'_t(\hat{b}_t^*) = 0$, we get
\begin{eqnarray*}
	u_t(b_t^*) - u_t(b_t) \leq u_t(\hat{b}_t^*) - u_t(b_t) \leq \frac{C}{2}(b_t - \hat{b}_t^*)^2 = \frac{C}{2} (\hat{b}_t - \hat{b}^*_t)^2,
\end{eqnarray*}
\end{itemize}
\end{itemize}

In summary, we have 
\begin{eqnarray*}
u_t(b_t^*) - u_t(b_t) \leq \frac{C}{2}(\hat{b}_t - \hat{b}^*_t)^2 \leq \frac{C}{2} |2 x_t \cdot (\alpha_0 - \hat{\alpha}_{s-1})|^2 \leq 2C |x_t \cdot (\alpha_0 - \hat{\alpha}_{s-1})|^2,
\end{eqnarray*}
where the second inequality holds because $(\varphi^{-1})'(x)\leq 1$ for all $x$.
\end{proof}

The following lemma is a technical lemma which is used to bound the regret in each episode, as shown in the proof for Theorem~\ref{thm:binary-known-noise-regret} later. The proof is technical and we leave it to Appendix~\ref{app:binary-feedback}.

\begin{lemma}\label{lem:second-order-concentration}
In each episode $s\geq 1$, we have
\begin{eqnarray*}
\frac{1}{T_s}\sum_{t\in \Gamma_s} \langle \hat{\alpha}_{s} - \alpha_0, x_t x_t^T(\hat{\alpha}_{s} - \alpha_0)\rangle \leq \frac{4h_W W}{\ell_W}\sqrt{\frac{\log(2dS/\delta)}{T_{s}}}
\end{eqnarray*}
holds with probability at least $1-\delta/2S$.
\end{lemma}

\begin{proof}
By second-order Taylor's theorem, we have
\begin{eqnarray*}
\mathcal{L}_{s}(\hat{\alpha}_{s}) - \mathcal{L}_{s}(\alpha_0) = \langle \nabla \mathcal{L}_s(\alpha_0), \hat{\alpha}_{s} - \alpha_0 \rangle + \frac{1}{2}\langle \hat{\alpha}_{s} - \alpha_0, \nabla^2\mathcal{L}_s(\tilde{\alpha})(\hat{\alpha}_{s} - \alpha_0)\rangle,
\end{eqnarray*}
for some $\tilde{\alpha}$ on the line segment between $\alpha_0$ and $\hat{\alpha}_s$. Given the definition of $\L_{s}(\alpha)$, we have
\begin{eqnarray}\label{eq:taylor-expansion-MLE}
\nabla\L_{s}(\alpha) = \frac{1}{T_s} \sum_{t\in \Gamma_s} \eta_t(\alpha)x_t, && \nabla^2\L_{s}(\tilde{\alpha}) = \frac{1}{T_s} \sum_{t\in \Gamma_s} \zeta_t(\tilde{\alpha})x_t x_t^T,
\end{eqnarray}
where $\eta_t(\alpha)$ and $\zeta_t(\alpha)$ are defined as follows,
\begin{eqnarray*}
\eta_t(\alpha) &=& -\log'F(b_t-\alpha\cdot x_t) \1\{m_t \leq b_t\} - \log'(1-F(b_t - \alpha\cdot x_t))\1\{m_t > b_t\}\\
\zeta_t(\alpha) &=& -\log''F(b_t-\alpha\cdot x_t) \1\{m_t \leq b_t\} - \log''(1-F(b_t - \alpha\cdot x_t))\1\{m_t > b_t\},
\end{eqnarray*}
Based on our construction of the algorithm, $x_t, b_t$ is independent with $z_t$. Thus, $b_t -\langle \alpha_0, x_t\rangle$ are independent with $z_t$ for any $t\in \Gamma_s$, then we have
\begin{eqnarray*}
\E[\eta_t(\alpha_0)] &=& \E[\E[\eta_t(\alpha_0)|x_t, b_t]]\\
& =& \E\left[-\frac{f(b_t - \alpha_0\cdot x_t)}{F(b_t-\alpha_0\cdot x_t)}\E[\1\{b_t \geq m_t\}|x_t, b_t] + \frac{f(b_t - \alpha_0\cdot x_t)}{1- F(b_t-\alpha_0\cdot x_t)}\E[\1\{b_t \leq m_t\}|x_t, b_t]\right]\\
&=& \E\left[-\frac{f(b_t - \alpha_0\cdot x_t)}{F(b_t-\alpha_0\cdot x_t)}F(b_t-\alpha_0\cdot x_t) + \frac{f(b_t - \alpha_0\cdot x_t)}{1- F(b_t-\alpha_0\cdot x_t)}(1 - F(b_t-\alpha_0\cdot x_t))\right]\\
&=& 0
\end{eqnarray*}  
Then by Hoeffding's inequality and union bound over each coordinate of $\nabla \L_s(\alpha_0)$.
\begin{eqnarray}\label{eq:binary-knonw-noise-infty-norm-gradient}
\Vert\nabla \L_{s}(\alpha_0)\Vert_\infty \leq 2h_W\sqrt{\frac{\log(2dS/\delta)}{T_{s}}}
\end{eqnarray}
holds with probability at least $1-\delta/2S$. By the optimality of $\hat{\alpha}_{s}$,
\begin{eqnarray}\label{eq:optimality-alpha-estimator}
\L_{s}(\hat{\alpha}_{s})  \leq \L_{s}(\alpha_0)
\end{eqnarray}
Invoking into Eq.~(\ref{eq:taylor-expansion-MLE}), we have
\begin{eqnarray*}
\frac{1}{2}\langle \hat{\alpha}_{s} - \alpha_0, \nabla^2\mathcal{L}_{s}(\tilde{\alpha})(\hat{\alpha}_{s} - \alpha_0)\rangle \leq - \langle \nabla \mathcal{L}_s(\alpha), \hat{\alpha}_{s} - \alpha_0 \rangle \leq \Vert\nabla \L_s(\alpha_0)\Vert_\infty \Vert\hat{\alpha}_s - \alpha_0\Vert_1
\end{eqnarray*}
In addition, by Assumption~\ref{assump:bounded-noise}, we have $\zeta_t(\tilde{\alpha}) \geq \ell_W$. Then the above inequality implies that
\begin{eqnarray*}
\frac{1}{T_s}\sum_{t\in \Gamma_s} \langle \hat{\alpha}_{s} - \alpha_0, x_t x_t^T(\hat{\alpha}_{s} - \alpha_0)\rangle \leq \frac{2}{\ell_W}\Vert\nabla \L_s(\alpha_0)\Vert_\infty \Vert\hat{\alpha}_s - \alpha_0\Vert_1 \leq \frac{4h_W W}{\ell_W}\sqrt{\frac{\log(2dS/\delta)}{T_{s}}}
\end{eqnarray*}
holds with probability at least $1-\frac{\delta}{2S}$, where the last inequality holds because $\Vert \hat{\alpha}_s - \alpha_0\Vert_1 \leq W$ and Eq.~(\ref{eq:binary-knonw-noise-infty-norm-gradient}).
\end{proof}

Given the above lemmas, we now turn to prove Theorem~\ref{thm:binary-known-noise-regret}.

\begin{proof}[Proof of Theorem~\ref{thm:binary-known-noise-regret}]
We first bound the total regret in each episode $s \geq 2$ in the following way,
\begin{eqnarray*}
\mathtt{Regret}_s &=& \sum_{t\in \Gamma_s}\E\left[u_t(b^*_t) - u_t(b_t)\right] \leq 2C \sum_{t\in \Gamma_s} \E\left[|x_t \cdot (\alpha_0 - \hat{\alpha}_{s-1})|^2\right]\\
&=& 2C \sum_{t\in \Gamma_s} \E\left[\langle \alpha_0 - \hat{\alpha}_{s-1}, x_t x_t^T (\alpha_0 - \hat{\alpha}_{s-1})\rangle\right]\\
&=& 2C \sum_{t\in \Gamma_s} \langle \alpha_0 - \hat{\alpha}_{s-1}, \Sigma (\alpha_0 - \hat{\alpha}_{s-1})\rangle,
\end{eqnarray*}
where $C = 2B_2 + B_3$. Then we decompose the term $\langle \alpha_0 - \hat{\alpha}_{s-1}, \Sigma (\alpha_0 - \hat{\alpha}_{s-1})\rangle$ in the following way,
\begin{eqnarray*}
&&\langle \alpha_0 - \hat{\alpha}_{s-1}, \Sigma (\alpha_0 - \hat{\alpha}_{s-1})\rangle\\
&=& \frac{1}{T_{s-1}} \sum_{t\in \Gamma_{s-1}}\langle \alpha_0 - \hat{\alpha}_{s-1}, x_t x^T_t (\alpha_0 - \hat{\alpha}_{s-1})\rangle + \langle \alpha_0 - \hat{\alpha}_{s-1},  E(\alpha_0 - \hat{\alpha}_{s-1})\rangle,
\end{eqnarray*}
where $E=\Sigma - \frac{1}{T_{s-1}} \sum_{t\in \Gamma_{s-1}}x_t x_t^T$. Then by Hoeffding's inequality and union bound over all indices $i,j \in [d]$, we have with probability at least $1-\frac{\delta}{2S}$, $\Vert E_{ij}\Vert \leq 3\sqrt{\frac{\log(2d^2S/\delta)}{T_{s-1}}}$ holds. Combining with Lemma~\ref{lem:second-order-concentration}, for any $s\geq 2$, we have
\begin{eqnarray*}
\mathtt{Regret}_s &\leq& 2C\left(\frac{4h_W W}{\ell_W}\sqrt{\frac{\log(2dS/\delta)}{T_{s-1}}} + 3W^2\sqrt{\frac{\log(2d^2S/\delta)}{T_{s-1}}}\right) T_s\\
&\leq& 8C(\frac{h_W W}{\ell_W} + W^2)\sqrt{\log(2d^2S/\delta)}\sqrt{T},
\end{eqnarray*}
holds with probability at least $1-\frac{\delta}{S}$. Therefore, by union bound over all stages $s=2,\cdots, S$, with probability at least $1-\delta$, the total regret is bounded by
\begin{eqnarray*}
R(T) &=& \sum_{t\in \Gamma_1}\E\left[u_t(b^*_t) - u_t(b_t)\right] +  \sum_{s=2}^S \mathtt{Regret}_s \leq \sqrt{T} + \sum_{s=2}^S \mathtt{Regret}_s
\leq O\left(S\sqrt{\log(2d^2S/\delta) T}\right),
\end{eqnarray*}
where the first inequality holds because $u_t(\cdot)$ is bounded by $[0, 1]$. Finally we bound $S$, it is easy to verify $S \leq \log\log T$. Thus, we complete the proof.
\end{proof}

\subsection{Proof of Theorem~\ref{thm:regret-binary-partial-known-noise}}\label{app:binary-partial-known-noise}
We restate Theorem~\ref{thm:regret-binary-partial-known-noise}.
\BinaryPartialNoise*

To begin with, it is straightforward to show the following two propositions, which can be directly derived from Assumption~\ref{assump:log-concave}.
\begin{proposition}\label{prop:binary-partial-known-log-concave}
The density function $f_0$ is differentiable and log-concave.
\end{proposition}

\begin{proposition}\label{prop:binary-partial-known-bounded-derivative}
There exists positive constants $h^0_W$, $\ell^0_W$ (depending on $W$) such that,
\begin{eqnarray*}
\max\{|\log'F_0(x)|, |\log'(1- F_0(x))|\} \leq h^0_W, \forall x \in [-W^2, W^2+W]
\end{eqnarray*}
\begin{eqnarray*}
\min\{-\log''F_0(x), -\log''(1- F_0(x))\} \geq \ell^0_W, \forall x \in [-W^2, W^2+W]
\end{eqnarray*}
\end{proposition}

To prove Theorem~\ref{thm:regret-binary-partial-known-noise}, we provide some auxiliary lemmas presented in the following. Lemma~\ref{eq:opt-bid-binary-partial-known-noise} provide a bound of the difference between the optimal expected utility and the expected utility achieved by our algorithm in at each time $t$ in episode $s$.

\begin{lemma}\label{lem:unknown-noise-each-time-regret}
For any $t\in \Gamma_s$, let $b_t^*$ be the optimal bid given context $x_t$ (Eq.~(\ref{eq:opt-bid-binary-partial-known-noise})),
\begin{eqnarray*}
u_t(b_t^*)  - u_t(b_t) \leq \frac{6C}{\rho_0^2} \left|\langle \hat{\mu}_{s-1}-\mu_0, x_t\rangle\right|^2 + \frac{8C(1+W)^2}{\rho_0^2}(\hat{\rho}_{s-1} - \rho_0)^2,
\end{eqnarray*}
with $C = 2B_1 + B_3$, where $B_1$ and $B_3$ are positive constants defined in Assumption~\ref{assump:bounded-noise}.
\end{lemma}
\begin{proof}
Based on the same argument in Lemma~\ref{lem:binary-known-noise-utility-difference}, we have $|u''_t(b)|\leq C$ with $C = 2B_1 + B_3$. Denote $\tilde{b}_t^* = \frac{1}{\rho_0} \cdot \left(\mu_0\cdot x_t + \varphi^{-1}_0(\rho_0\beta_0(x_t) - \mu_0\cdot x_t)\right)$ and $\tilde{b}_t = \frac{1}{\hat{\rho}_{s-1}}\cdot \left(\hat{\mu}_{s-1}\cdot x_t + \varphi^{-1}_0(\hat{\rho}_{s-1}\beta_0(x_t) - \hat{\mu}_{s-1}\cdot x_t)\right)$.

Then, by the similar case analysis used in Lemma~\ref{lem:binary-known-noise-utility-difference}, we can bound $u_t(b_t^*) - u_t(b_t) \leq \frac{C}{2}(\tilde{b}_t^* - \tilde{b}_t)^2$. Thus, we have
\begin{eqnarray*}
&&u_t(b_t^*) - u_t(b_t) \leq \frac{C}{2}(\tilde{b}_t^* - \tilde{b}_t)^2\\
&\leq& \frac{C}{2} \left|\langle \frac{\hat{\mu}_{s-1}}{\hat{\rho}_{s-1}} - \frac{\mu_0}{\rho_0}, x_t\rangle + \frac{1}{\hat{\rho}_{s-1}} \varphi^{-1}_0(\hat{\rho}_{s-1}\beta_0(x_t) - \langle \mu_{s-1}, x_t\rangle) - \frac{1}{\rho_0}\varphi^{-1}_0(\rho_0\beta_0(x_t) - \langle \mu_0, x_t\rangle)\right|^2\\
&\leq& C\left|\frac{1}{\rho_0}\langle \hat{\mu}_{s-1}-\mu_0, x_t\rangle + \left(\frac{1}{\hat{\rho}_{s-1}} - \frac{1}{\rho_0}\right)\langle \hat{\mu}_{s-1}, x_t\rangle\right|^2 + \\
&& C\Bigg|\frac{1}{\rho_0}\left(\varphi^{-1}_0(\hat{\rho}_{s-1}\beta_0(x_t) - \langle \mu_{s-1}, x_t\rangle) - \varphi^{-1}_0(\rho_0\beta_0(x_t) - \langle \mu_0, x_t\rangle)\right) + \\
&& \left(\frac{1}{\hat{\rho}_{s-1}} - \frac{1}{\rho_0}\right)\varphi^{-1}_0(\hat{\rho}_{s-1}\beta_0(x_t) - \langle \mu_{s-1}, x_t\rangle)\Bigg|^2\\
&\leq& \frac{2C}{\rho_0^2} \left|\langle \hat{\mu}_{s-1}-\mu_0, x_t\rangle\right|^2 + \frac{2C}{\rho_0^2}(\hat{\rho}_{s-1} - \rho_0)^2\left|\langle \frac{\hat{\mu}_{s-1}}{\hat{\rho}_{s-1}}, x_t\rangle\right|^2 + \frac{2C}{\rho_0^2} \left|(\hat{\rho}_{s-1} -\rho_0)\beta_0(x_t) + \langle \hat{\mu}_{s-1} - \mu_0, x_t\rangle\right|^2\\
&& + \frac{2C}{\rho_0^2} (\hat{\rho}_{s-1} - \rho_0)^2 \left|\beta_0(x_t) - \left\langle\frac{\hat{\mu}_{s-1}}{\hat{\rho}_{s-1}}, x_t \right\rangle\right|^2\\
&\leq&\frac{2C}{\rho_0^2} \left|x_t\cdot(\hat{\mu}_{s-1}-\mu_0)\right|^2  + \frac{2CW^2}{\rho_0^2}(\hat{\rho}_{s-1} - \rho_0)^2 + \frac{4C}{\rho_0^2} (\hat{\rho}_{s-1} - \rho_0)^2 + \frac{4C}{\rho_0^2}\left|\langle \hat{\mu}_{s-1}-\mu_0, x_t\rangle\right|^2\\
&& + \frac{2C(1+W)^2}{\rho_0^2} (\hat{\rho}_{s-1} - \rho_0)^2\\
&\leq & \frac{6C}{\rho_0^2} \left|x_t\cdot(\hat{\mu}_{s-1}-\mu_0)\right|^2 + \frac{8C(1+W)^2}{\rho_0^2}(\hat{\rho}_{s-1} - \rho_0)^2
\end{eqnarray*}
where the second and third inequalities hold because Cauchy-Schwartz and the fact that $|\varphi_0^{-1}(x) - \varphi_0^{-1}(y)| \leq |x-y|$. The fourth inequality holds because $\beta_0(x_t) \leq 1$ and $\Vert \frac{\hat{\mu}_{s-1}}{\hat{\rho}_{s-1}}\Vert_1\leq W$.
\end{proof}

Given the above lemma, to bound the regret, we need to bound $\left|\langle \hat{\mu}_{s}-\mu_0, x_t\rangle\right|^2$ and $(\hat{\rho}_{s} - \rho_0)^2$ simultaneously in each episode $s$. First, we show $\L_s(\mu, \rho)$ is $\gamma$-strongly convex with high probability in Lemma~\ref{lem:binary-partial-known-noise-strong-convexity}. %

\begin{lemma}\label{lem:binary-partial-known-noise-strong-convexity}
Suppose Assumption~\ref{assump:eigenvalue-ub} holds.
For any $\delta \in (0, 1)$, with probability at least $1-\frac{\delta}{2S}$, $\L_s(\cdot, \cdot)$ is $\gamma$-strongly almost everywhere, where%
$\gamma = \ell^0_W \lambda^* / 2$ when $T_s$ is sufficiently large such that, $\sqrt{\frac{\log(2S(d+1)/\delta)}{2T_s}} \leq \frac{\lambda^*}{2}$.

$\lambda^*$ is a constant, s.t.,
\begin{eqnarray*}
\lambda^* = \min\left\{\lambda_2, \lambda_3/2, \Delta^2 c_x \left(\frac{2}{\lambda_3} + \frac{1}{\lambda_0}\right)\right\},
\end{eqnarray*}
where $c_x = \E[x_t]^T E[x_t]$, $\lambda_0 > 0$ is the minimum eigenvalue of matrix $\E[x_t]\E[x_t]^T$, $\lambda_2$ and $\lambda_3$ are the parameters defined in Assumption~\ref{assump:eigenvalue-ub}.
\end{lemma}

\begin{proof}
Let $\tilde{x_t} = (x_t; -b_t)$. First, we show the second-order derivative of $\L_s(\mu, \rho)$ in the following, 
\begin{eqnarray*}
\nabla^2\L_s(\mu, \rho) = \frac{1}{T_s} \sum_{t\in \Gamma_s} \zeta_t(\mu, \rho) \tilde{x_t} \tilde{x_t}^T, 
\end{eqnarray*}
where $\zeta_t(\mu, \rho)$ is defined as below,
\begin{eqnarray*}
\zeta_t(\mu, \rho) = -\log''F_0(\rho b_t- \mu \cdot x_t) \1\{m_t \leq b_t\} - \log''(1-F_0(\rho b_t - \mu\cdot x_t))\1\{m_t > b_t\},
\end{eqnarray*}

By our assumption, we have $\zeta_t(\mu, \rho) \geq \ell^0_W$, for any $(\mu, \rho)\in \Lambda$. 
$$
\tilde{x_t} \tilde{x_t}^T = 
\begin{bmatrix}
x_t x_t^T & - b_t x_t\\
- b_t x_t^T & b_t^2
\end{bmatrix}
$$

Then we have 
$$
\E\left[\tilde{x_t} \tilde{x_t}^T\right] = 
\begin{bmatrix}
\Sigma & - \E[b_t x_t]\\
- \E[b_t x_t]^T & \E[b_t^2]
\end{bmatrix}
$$
For notation simplicity, we denote $c_x = \E[x_t]^T \E[x_t] > 0$, $\bar{x} = \E[x_t]$, and $\bar{\Sigma} = \E[x_t]\E[x_t]^T$. Denote the minimum eigenvalue of $\bar{\Sigma}$ as $\lambda_0 > 0$.
We set 
\begin{eqnarray}
\lambda^* = \min\left\{\lambda_2, \lambda_3/2, \Delta^2 c_x \left(\frac{2}{\lambda_3} + \frac{1}{\lambda_0}\right)\right\}.
\end{eqnarray}
Next we would like to prove $\E\left[\tilde{x_t} \tilde{x_t}^T\right] \succeq \lambda^* I$ with high probability, where $I$ is identity matrix. Denote matrix
$$
M = 
\begin{bmatrix}
\Sigma - \lambda^*I  & - \E[b_t x_t]\\
- \E[b_t x_t]^T & \E[b_t^2] - \lambda^*
\end{bmatrix}
$$
Then $\E\left[\tilde{x_t} \tilde{x_t}^T\right] \succeq \lambda^* I$ is equivalent to $M \succeq 0$. By Schur Complements (Lemma~\ref{lem:schur-complement}), it is equivalent to show matrix $M' \succeq 0$, where $M'$ is defined as, 
$$
M' = 
\begin{bmatrix}
\Sigma - \lambda^*I  & 0\\
0 & \E[b_t^2] - \lambda^* - \E[b_t x_t]^T (\Sigma - \lambda^*I)^{-1} \E[b_t x_t]
\end{bmatrix}
$$
 
Then we have,
\begin{eqnarray*}
&&\E[b_t^2] - \lambda^* - \E[b_t x_t]^T (\Sigma - \lambda^*I)^{-1} \E[b_t x_t]\\
&\geq & \Delta^2 \cdot \left(1 - \bar{x}^T (\bar{\Sigma} + \lambda_3 I - \lambda^*I)^{-1} \bar{x}\right) - \lambda^* \\
&& (\text{Because } b_t \geq \Delta, \mbox{and } \Sigma \succeq \bar{\Sigma} + \lambda_3 I \mbox{ in Assumption~\ref{assump:eigenvalue-ub}}.)\\
&\geq& \Delta^2 \cdot \left(1 - \bar{x}^T \left(\bar{\Sigma} + \frac{\lambda_3}{2} I\right)^{-1} \bar{x}\right) - \lambda^*\\
&& (\mbox{Because } \lambda^* \leq \frac{\lambda_3}{2})\\
&\geq& \Delta^2 \cdot \left(1 - \bar{x}^T \bar{\Sigma}^{-1} \bar{x} + \bar{x}^T \bar{\Sigma}^{-1} \left(\frac{2}{\lambda_3} I + \bar{\Sigma}^{-1}\right)^{-1} \bar{\Sigma}^{-1} \bar{x}\right) - \lambda^*\\
&& (\mbox{By Matrix Inverse Lemma (Lemma~\ref{lem:matrix-inverse})})\\
&\geq & \Delta^2 \cdot \left(\left(\frac{2}{\lambda_3} + \frac{1}{\lambda_0}\right)\bar{x}^T \bar{\Sigma}^{-1}\bar{\Sigma}^{-1} \bar{x}\right) - \lambda^*\\
&& \left(\mbox{Since } \bar{x}^T \bar{\Sigma}^{-1} \bar{x} = 1, \mbox{ and }\frac{2}{\lambda_3} I + \bar{\Sigma}^{-1} \preceq \left(\frac{2}{\lambda_3} + \frac{1}{\lambda_0}\right) I \right)\\
&\geq& \Delta^2 c_x \left(\frac{2}{\lambda_3} + \frac{1}{\lambda_0}\right) - \lambda^* \geq 0
\end{eqnarray*}

On the other hand, it is trivial to show $\Sigma - \lambda^*I \succeq 0$.
Therefore, we prove $M'\succeq 0$, which is equivalent to $\E[x_t]\E[x_t]^T \succeq \lambda^* I$.
In addition, $\lambda_{\max}(\tilde{x}_t \tilde{x}_t^T) \leq 1$ since $\Vert x_t\Vert_\infty \leq 1$ and $b_t \leq 1$ for any time $t$. Then by Matrix Chernoff bound, we have,
\begin{eqnarray*}
\PP\left(\frac{1}{T_s} \sum_{t\in \Gamma_s} \tilde{x}_t \tilde{x}_t^T \succeq (\lambda^* - \eps)\cdot I\right) &\geq& 1 - \PP\left(\lambda_{\min}\left(\frac{1}{T_s} \sum_{t\in \Gamma_s} \tilde{x}_t \tilde{x}_t^T\right)\leq \lambda^* - \eps\right)\\
&\geq& 1 - (d+1)\cdot e^{-2\eps^2 T_s}
\end{eqnarray*}

Setting $\delta = 2S(d+1)e^{-2\eps^2 T_s}$, we have, with probability at least $1-\delta/2S$, $\L_s(\cdot, \cdot)$ is\\ $\ell^0_W \left(\lambda^* - \sqrt{\frac{\log(2S(d+1)/\delta)}{2T_s}}\right)$-strongly almost everywhere. When $T_s$ is sufficiently large such that $\sqrt{\frac{\log(2S(d+1)/\delta)}{2T_s}} \leq \frac{\lambda^*}{2}$.
\end{proof}

Given the strong convexity of $\L_s(\mu, \rho)$, we can bound the $L_2$ distance $\Vert (\hat{\mu}_s, \hat{\rho}_s) -  (\mu_0, \rho_0)\Vert_2^2$ in the following lemma,

\begin{lemma}\label{lem:unknown-noise-L2-convergence}
Suppose Assumption~\ref{assump:eigenvalue-ub} holds. For any $\delta \in (0, 1)$, we have
\begin{eqnarray*}
\Vert (\hat{\mu}_s, \hat{\rho}_s) -  (\mu_0, \rho_0)\Vert_2^2 \leq \frac{4h(W^2+W)}{\gamma} \sqrt{\frac{\log(2dS/\delta)}{T_s}}
\end{eqnarray*}
holds with probability at least $1 - \frac{\delta}{S}$, where $\gamma$ is defined in the statement of Lemma~\ref{lem:binary-partial-known-noise-strong-convexity}.
\end{lemma}
\begin{proof}
By Lemma~\ref{lem:binary-partial-known-noise-strong-convexity}, with probability at least $1-\frac{\delta}{2S}$, $\L_s(\cdot,\cdot)$ is $\gamma$-strongly convex at $(\mu_0, \rho_0)$. Then we have, with probability at least $1-\frac{\delta}{2S}$,
\begin{eqnarray}\label{eq:unknown-noise-strongly-convex}
\L_s(\hat{\mu}_s, \hat{\rho}_s) \geq \L_s(\mu_0, \rho_0) + \langle\nabla\L_s(\mu_0, \rho_0), (\hat{\mu}_s, \hat{\rho}_s) - (\mu_0, \rho_0)\rangle + \frac{\gamma}{2}\Vert (\hat{\mu}_s, \hat{\rho}_s) -  (\mu_0, \rho_0)\Vert_2^2
\end{eqnarray}
Then we have
\begin{equation}\label{eq:binary-partial-known-noise-l2-bound}
\begin{aligned}
\frac{\gamma}{2}\Vert (\hat{\mu}_s, \hat{\rho}_s) -  (\mu_0, \rho_0)\Vert_2^2 &\leq \Vert \nabla\L_s(\mu_0, \rho_0)\Vert_\infty \cdot \Vert (\hat{\mu}_s, \hat{\rho}_s) - (\mu_0, \rho_0)\Vert_1 + \L_s(\hat{\mu}_s, \hat{\rho}_s) - \L_s(\mu_0, \rho_0)\\
&\leq \Vert \nabla\L_s(\mu_0, \rho_0)\Vert_\infty \cdot \Vert (\hat{\mu}_s, \hat{\rho}_s) - (\mu_0, \rho_0)\Vert_1,
\end{aligned}
\end{equation}
holds with probability at least $1-\frac{\delta}{2S}$. By the same argument in Lemma~\ref{lem:second-order-concentration} (Eq.~(\ref{eq:binary-knonw-noise-infty-norm-gradient})), with probability at least $1-\frac{\delta}{2S}$, we have
\begin{eqnarray*}
\Vert \nabla\L_s(\mu_0, \rho_0)\Vert_\infty \leq 2h^0_W(W^2+W) \sqrt{\frac{\log(2dS/\delta)}{T_s}}
\end{eqnarray*}
Combining Eq.~(\ref{eq:binary-partial-known-noise-l2-bound}) and union bound, we have
\begin{eqnarray*}
\Vert (\hat{\mu}_s, \hat{\rho}_s) -  (\mu_0, \rho_0)\Vert_2^2 \leq \frac{4h^0_W(W^2+W)}{\gamma} \sqrt{\frac{\log(2dS/\delta)}{T_s}}
\end{eqnarray*}
holds with probability at least $1 - \frac{\delta}{S}$.
\end{proof}

Given the above technical lemmas, we are ready to prove Theorem~\ref{thm:regret-binary-partial-known-noise} shown as below,

\begin{proof}[Proof of Theorem~\ref{thm:regret-binary-partial-known-noise}]
Let $b_t^* = \argmax_b u(b, x_t)$. Then the regret achieved in each episode $s$ can be represented as follows,
\begin{eqnarray*}
\rgt_s &=& \sum_{t\in \Gamma_s}\E_{x_t}\left[u_t(b^*_t) - u_t(b_t)\right]\\
&\leq& \sum_{t\in \Gamma_s} \E_{x_t}\left[\frac{6C}{\rho_0^2} \left|\langle \hat{\mu}_{s-1}-\mu_0, x_t\rangle\right|^2 + \frac{8C(B+W)^2}{\rho_0^2}(\hat{\rho}_{s-1} - \rho_0)^2\right]\\
&\leq& \sum_{t\in \Gamma_s} \frac{6C\lambda_{1}}{\rho_0^2} \Vert \hat{\mu}_{s-1}-\mu_0\Vert_2^2 + \frac{8C(B+W)^2}{\rho_0^2}(\hat{\rho}_{s-1} - \rho_0)^2\\
&\leq& \frac{C'}{\rho_0^2} \sum_{t\in \Gamma_{s}} \Vert (\hat{\mu}_{s-1}, \hat{\rho}_{s-1}) -  (\mu_0, \rho_0)\Vert_2^2,
\end{eqnarray*}
for some constant $C'$ depending on $C, \lambda_{1}, B, W$. The first inequality is because of Lemma~\ref{lem:unknown-noise-each-time-regret} and the second inequality holds because the Assumption~\ref{assump:eigenvalue-ub} holds.

Finally, by Lemma~\ref{lem:unknown-noise-L2-convergence}, we can bound $\rgt_s$ for each stage $s$,
\begin{eqnarray*}
\rgt_s \leq \left(\frac{4h^0_W(W^2+W)C'}{\gamma\rho_0^2}\sqrt{\frac{\log(dS/\delta)}{T_{s-1}}}\right) T_s \leq \frac{4h^0_W(W^2+W)C'}{\gamma\rho_0^2}\sqrt{\log(dS/\delta)}\sqrt{T},
\end{eqnarray*}
holds with probability at least $1-\frac{\delta}{S}$, where $\gamma$ is defined in the statement of Lemma~\ref{lem:binary-partial-known-noise-strong-convexity}. Therefore, by union bound over all stages $s=1,\cdots, S$, with probability at least $1-\delta$, the total regret is bounded by
\begin{eqnarray*}
R(T) \leq \sum_{s=1}^S \rgt_s \leq O\left(S\sqrt{\log(dS/\delta) T}\right)\leq O\left(\log\log T \sqrt{\log(d\log\log T/\delta) T}\right),
\end{eqnarray*}
the last inequality holds because $S\leq \log\log T$.
\end{proof}

\section{Omitted Proofs from Section~\ref{sec:full-unknown-noise}}

\subsection{Proof of Theorem~\ref{thm:cdf-convergence}}\label{app:cdf-convergence}
\CDFConvergence*
\begin{proof}
Denote $r_s = \frac{1}{2B_1 \sqrt{T_s}}$. For any $z\in [-W, 1+W]$, the probability that there exists a point $y\in \{z_t\}_{t\in \Gamma_s}$ such that $|y - z| \leq r_s$, is at least
\begin{eqnarray*}
1 - (1 - 2B_1\cdot r_s)^{T_s} = 1 - \left(1 - \frac{1}{\sqrt{T_s}}\right)^{T_s} \approx 1- e^{-\sqrt{T_s}}
\end{eqnarray*}
Then for any $z$ and $y$ s.t. $|y - z|\leq r_s$, we can decompose $\vert\hat{F}_s(z; \alpha_0) - F(z)\vert$ in the following,
\begin{eqnarray*}
&&\vert\hat{F}_s(z; \alpha_0) - F(z)\vert\\
&\leq & \vert\hat{F}_s(z; \alpha_0) - \hat{F}_s(y; \alpha_0)\vert + \vert \hat{F}_s(y; \alpha_0) - \mathbb{F}_s(y)\vert + \vert \mathbb{F}_s(y) - F(y)\vert + \vert F(y) - F(z)\vert\\
&\leq& B_2 \cdot r_s + \frac{1}{T_s} + \sqrt{\frac{\log(4/\delta)}{2T_s}} + B_2\cdot r_s\\
& = & \frac{B_2}{B_1\sqrt{T_s}} + \sqrt{\frac{\log(4/\delta)}{2T_s}} + \frac{1}{T_s} \\
\end{eqnarray*}
The second inequality follows from  $\hat{F}_s(\cdot; \alpha_0)$ is $B_2$-Lipschitz, $F$ is $B_2$-Lipschitz, Dvoretzky-Kiefer-Wolfowitz (DKW) inequality and Lemma~\ref{lem:empirical-cdf-knots}. It holds with probability at least $1- \delta/2 - e^{-\sqrt{T_s}}$. Since $T_s \gg \log^2(2/\delta)$, we complete the proof.
\end{proof}

\subsection{Proof of Theorem~\ref{thm:full-info-regret}}\label{app:full-info-regret}
Firstly, given Theorem~\ref{thm:cdf-convergence}, $\hat{F}_s(\cdot; \alpha_0)$ is arbitrarily close to $F$ when $T_s$ is sufficiently large. In addition, \citet{DR09} also show $\hat{f}_s(\cdot; \alpha_0)$ is arbitrarily close to $f$ when $T_s$ is sufficiently large. Therefore, we can show,

\begin{proposition}\label{prop:constant-bound-log-derivative}
There exists positive constants $\tilde{h}_W$ and $\tilde{\ell}_W$ (depending on $W$), such that
\begin{eqnarray*}
\max\{|\log'\hat{F}_s(x; \alpha_0)|, |\log'(1- \hat{F}_s(x; \alpha_0))|\} \leq \tilde{h}_W, \forall x \in [-W, 1+W]
\end{eqnarray*}
and
\begin{eqnarray*}
\min\{-\log''\hat{F}_s(x; \alpha), -\log''(1- \hat{F}_s(x; \alpha))\} \geq \tilde{\ell}_W, \forall x \in [-W, 1+W], \forall \alpha: \Vert\alpha\Vert_1 \leq W
\end{eqnarray*}
both hold almost surely.
\end{proposition}

Let $X_s \in \R^{T_s\times d}$ be the context matrix with rows $x_t, t\in\Gamma_s$, corresponding to $T_s$ auctions in episode $s$. Then we have the following Proposition, provided by~\cite{Bhlmann2011},
\begin{proposition}\label{prop:eigenvalue-lb-empirical-matrix}
Let $\Sigma_s = (X^T_s X_s)/T_s$ and 
$T_s$ is sufficiently large. Under Assumption~\ref{assump:eigenvalue-ub}, the eigenvalue of $\Sigma_s$ is at least $\sqrt{\lambda_2/2}$ almost surely.
\end{proposition}

To prove Theorem~\ref{thm:full-info-regret}, we first propose several technical lemmas shown as below,
\begin{lemma}\label{lem:full-info-convergence}
For each stage $s$, we have
\begin{eqnarray*}
\Vert \hat{\alpha}_s - \alpha_0\Vert_2 \leq \frac{4A\sqrt{d}}{\tilde{\ell}_W\lambda_2}
\end{eqnarray*}
holds with probability at least $1-\delta/2S$, where $A = 2\tilde{h}_W \cdot \left(\frac{B_2}{B_1\sqrt{T_s}} + \sqrt{\frac{\log(16S/\delta)}{2T_s}} + \sqrt{\frac{\log(2dS/\delta)}{T_{s}}} + \frac{1}{T_s}\right)$.
\end{lemma}

\begin{proof}
By the second-order Taylor theorem, we have
\begin{eqnarray*}
\mathcal{L}_{s}(\hat{\alpha}_{s}) - \mathcal{L}_{s}(\alpha_0) = \langle \nabla \mathcal{L}_s(\alpha_0), \hat{\alpha}_{s} - \alpha_0 \rangle + \frac{1}{2}\langle \hat{\alpha}_{s} - \alpha_0, \nabla^2\mathcal{L}_s(\tilde{\alpha})(\hat{\alpha}_{s} - \alpha_0)\rangle,
\end{eqnarray*}
for some $\tilde{\alpha}$ on the line segment between $\alpha$ and $\hat{\alpha}_s$. Given the definition of $\L_{s}(\alpha; \kappa_s)$, we have
\begin{eqnarray}\label{eq:taylor-expansion-censored}
\nabla\L_{s}(\alpha) = \frac{1}{T_s} \sum_{t\in \Gamma_s} \eta_t(\alpha)x_t, && \nabla^2\L_{s}(\tilde{\alpha}) = \frac{1}{T_s} \sum_{t\in \Gamma_s} \zeta_t(\tilde{\alpha})x_t x_t^T,
\end{eqnarray}
where $\eta_t(\alpha)$ and $\zeta_t(\alpha)$ are defined as follows,
\begin{eqnarray*}
\eta_t(\alpha) &=& -\log'\hat{F}_s(b_t-\alpha\cdot x_t; \alpha) \1\{m_t \leq b_t\} - \log'(1-\hat{F}_s(b_t - \alpha\cdot x_t; \alpha))\1\{m_t > b_t\}\\
\zeta_t(\alpha) &=& -\log''\hat{F}_s(b_t-\alpha\cdot x_t; \alpha) \1\{m_t \leq b_t\} - \log''(1-\hat{F}_s(b_t - \alpha\cdot x_t; \alpha))\1\{m_t > b_t\},
\end{eqnarray*}

Based on our construction of the algorithm, $x_t, b_t$ is independent with $z_t$. Therefore, $\eps_t(\alpha_0) = b_t -\langle \alpha_0, x_t\rangle$ are independent with $z_t$ for any $t\in \Gamma_s$, we have
\begin{eqnarray*}
\E[\eta_t(\alpha_0)] &=& -\frac{ \hat{f}_s(\eps_t(\alpha_0); \alpha_0)}{\hat{F}_s(\eps_t(\alpha_0); \alpha_0)}\cdot F(\eps_t(\alpha_0)) + \frac{ \hat{f}_s(\eps_t(\alpha_0); \alpha_0)}{1- \hat{F}_s(\eps_t(\alpha_0); \alpha_0)}\cdot \left(1 - F(\eps_t(\alpha_0))\right)\\
&=& \left[\hat{F}_s(\eps_t(\alpha_0); \alpha_0) - F(\eps_t(\alpha_0))\right] \cdot \left[\frac{ \hat{f}_s(\eps_t(\alpha_0); \alpha_0)}{\hat{F}_s(\eps_t(\alpha_0); \alpha_0)} + \frac{\hat{f}_s(\eps_t(\alpha_0); \alpha_0)}{1- \hat{F}_s(\eps_t(\alpha_0); \alpha_0)}\right].
\end{eqnarray*}

Thus, by Theorem~\ref{thm:cdf-convergence} and Proposition~\ref{prop:constant-bound-log-derivative}, we have 
\begin{eqnarray*}
\vert\E[\eta_t(\alpha_0)]\vert &\leq&\left\vert\hat{F}_s(\eps_t(\alpha_0); \alpha_0) - F(\eps_t(\alpha_0))\right\vert \cdot \left\vert\frac{\hat{f}_s(\eps_t(\alpha_0); \alpha_0)}{\hat{F}_s(\eps_t(\alpha_0); \alpha_0)} + \frac{\hat{f}_s(\eps_t(\alpha_0); \alpha_0)}{1- \hat{F}_s(\eps_t(\alpha_0); \alpha_0)}\right\vert\\
&\leq& 2\tilde{h}_W \cdot \left(\frac{B_2}{B_1\sqrt{T_s}} + \sqrt{\frac{\log(16S/\delta)}{2T_s}} + \frac{1}{T_s}\right)\\
\end{eqnarray*}
holds with probability at least $1-\delta/4S$. Then by Hoeffding's inequality and union bound
\begin{eqnarray*}
\Vert\nabla \L_{s}(\alpha_0)\Vert_\infty &\leq& \frac{1}{T_s}\sum_{t\in\Gamma_s}\vert\E[\eta_t(\alpha_0)]\vert +  2\tilde{h}_W\sqrt{\frac{\log(2dS/\delta)}{T_{s}}}\\
&\leq&  2\tilde{h}_W \cdot \left(\frac{B_2}{B_1\sqrt{T_s}} + \sqrt{\frac{\log(16S/\delta)}{2T_s}} + \frac{1}{T_s}\right) + 2\tilde{h}_W\sqrt{\frac{\log(2dS/\delta)}{T_{s}}} := A
\end{eqnarray*}
holds with probability at least $1-\delta/2S$. By the optimality of $\hat{\alpha}_{s}$,
\begin{eqnarray*}
\L_{s}(\hat{\alpha}_{s})  \leq \L_{s}(\alpha_0)
\end{eqnarray*}
Invoking into Eq.~(\ref{eq:taylor-expansion-MLE}), we have
\begin{eqnarray*}
\frac{1}{2}\langle \hat{\alpha}_{s} - \alpha_0, \nabla^2\mathcal{L}_{s}(\tilde{\alpha})(\hat{\alpha}_{s} - \alpha_0)\rangle &\leq& - \langle \nabla \mathcal{L}_s(\alpha), \hat{\alpha}_{s} - \alpha_0 \rangle \leq \Vert\nabla \L_s(\alpha_0)\Vert_\infty \Vert\hat{\alpha}_s - \alpha_0\Vert_1 \\
&\leq& A \sqrt{d}\Vert \hat{\alpha}_s - \alpha_0\Vert_2
\end{eqnarray*}
In addition, by Proposition~\ref{prop:constant-bound-log-derivative}, we have $\zeta_t(\tilde{\alpha}) \geq \tilde{\ell}_W$.
Recall $X_s$ represent the context matrix with rows $x_t, t\in\Gamma_s$, corresponding to $T_s$ auctions in episode $s$.
Then the above inequality implies that
\begin{eqnarray*}
\frac{2\tilde{\ell}_W}{T_s} \Vert X_s (\hat{\alpha}_s - \alpha_0) \Vert^2_2 &\leq& 4A\sqrt{d}\Vert \hat{\alpha}_s - \alpha_0\Vert_2 \leq \frac{4A\sqrt{2d}}{\sqrt{T_s\lambda_2}}\Vert X_s (\hat{\alpha}_s - \alpha_0) \Vert_2\\
&\leq& \frac{\tilde{\ell}_W}{T_s}\Vert X_s (\hat{\alpha}_s - \alpha_0) \Vert^2_2 + \frac{8A^2 d}{\tilde{\ell}_W \lambda_2}
\end{eqnarray*}
holds with probability at least $1-\frac{\delta}{2S}$. %
The second inequality holds because Proposition~\ref{prop:eigenvalue-lb-empirical-matrix}\footnote{Indeed, with sufficiently large $T_s$, the eigenvalue of $\Sigma_s = (X_s^T X_s)/T_s$ is at least $\sqrt{\lambda_2/2}$ holds with probability at least $1-e^{-cT_s}$. Here we ignore the uncertainty of this event to simplify presentation.} and the third inequality is based on Cauchy-Schwartz inequality.

Thus, we have 
\begin{eqnarray*}
\frac{\lambda_2}{2} \Vert \hat{\alpha}_s - \alpha_0\Vert_2^2 \leq \frac{1}{T_s}\Vert X_s (\hat{\alpha}_s - \alpha_0) \Vert^2_2  \leq \frac{8 A^2 d}{\tilde{\ell}_W^2 \lambda_2},
\end{eqnarray*}
which implies $\Vert \hat{\alpha}_s - \alpha_0\Vert_2 \leq \frac{4A\sqrt{d}}{\tilde{\ell}_W\lambda_2}$.
\end{proof}

\begin{lemma}\label{lem:empirical-cdf-convergence}
For any fixed $\delta > 0$, suppose $T_s \gg \log^2(2/\delta)$ and
conditioned on $\Vert \hat{\alpha}_s - \alpha_0\Vert_2 \leq \kappa_s$, we have for all $z\in [-W, 1+W]$,
\begin{eqnarray*}
|\hat{F}_{s}(z;\hat{\alpha}_s) - F(z)|\leq 3B_2\kappa_s + \frac{B_2}{B_1\sqrt{T_s}} + \sqrt{\frac{\log(8/\delta)}{2T_s}} + \frac{1}{T_s}
\end{eqnarray*}
holds with probability at least $1-\delta$.
\end{lemma}

\begin{proof}
Let $\hat{\mathbb{F}}_s$ be the empirical distribution of samples $\{m_t - \hat{\alpha}_s\cdot x_t\}_{t\in \Gamma_s}$, i.e.
\begin{eqnarray*}
\hat{\mathbb{F}}_s(z) = \frac{1}{T_s}\sum_{t\in \Gamma_s} \1\left\{m_t - \langle \hat{\alpha}_s, x_t\rangle \leq z\right\} = \frac{1}{T_s}\sum_{t\in \Gamma_s} \1\left\{z_t \leq z + \langle \hat{\alpha}_s - \alpha_0, x_t \rangle \right\}
\end{eqnarray*}
First, we give a uniform convergence bound for $\vert \hat{\mathbb{F}}_s(z) - F(z)\vert$. The proof is analogous to the proof of Lemma 1 in~\citep{Golrezaei2019}. The main challenge is that we cannot directly apply DKW inequality, since $\hat{\alpha}_s$ depends on $z_t, t\in \Gamma_s$. To handle this challenge, we bound the lower bound and upper bound of $\hat{\mathbb{F}}_s(z)$ separately.

Since $\Vert \hat{\alpha}_s - \alpha_0\Vert_2 \leq \kappa_s$ and $\Vert x_t\Vert_2 \leq 1$, we have
\begin{eqnarray*}
\frac{1}{T_s}\sum_{t\in \Gamma_s}\1\left\{z_t \leq z -\kappa_s \right\} \leq \hat{\mathbb{F}}_s(z) \leq \frac{1}{T_s}\sum_{t\in \Gamma_s}\1\left\{z_t \leq z +\kappa_s \right\}
\end{eqnarray*}
Thus, conditioned on $\Vert \hat{\alpha}_s - \alpha_0\Vert_2 \leq \kappa_s$, for any $\gamma > 0$, we have
\begin{eqnarray*}
&&\PP\left(\hat{\mathbb{F}}_s(z) - F(z + \kappa_s) \leq \gamma\right)\\
&\geq & \PP\left(\frac{1}{T_s}\sum_{t\in \Gamma_s}\1\left\{z_t \leq z +\kappa_s \right\} - F(z + \kappa_s) \leq \gamma\right)\\
&\geq& 1 - \PP\left(\sup_{z}\left\vert\frac{1}{T_s}\sum_{t\in \Gamma_s}\1\left\{z_t \leq z +\kappa_s \right\} - F(z + \kappa_s)\right\vert > \gamma\right)\\
&\geq& 1 - 2\exp(-2T_s \gamma^2)
\end{eqnarray*}
Similarly, we have $\PP\left(F(z-\kappa_s) - \hat{\mathbb{F}}_s(z) \leq \gamma\right) \geq 1 - 2\exp(-2T_s\gamma^2)$ for any $\gamma> 0$, conditioned on $\Vert \hat{\alpha}_s - \alpha_0\Vert_2 \leq \kappa_s$.

Therefore, applying a union bound and Lipschitzness of $F$ yields,
\begin{eqnarray}\label{eq:empirical-cdf-convergence-episode}
\vert \hat{\mathbb{F}}_s(z) - F(z)\vert \leq \sqrt{\frac{\log(8/\delta)}{2T_s}} + B_2\kappa_s
\end{eqnarray}
holds with probability at least $1-\delta/2$.

Second, we apply the similar technique used in Theorem~\ref{thm:cdf-convergence} to bound $|\hat{F}_{s}(z;\hat{\alpha}_s) - F(z)|$. Denote $\hat{z}_t = m_t - \hat{\alpha}_s\cdot x_t, \forall t\in \Gamma_s$. Thus, for any $\hat{z}_t$, there must exist at least one $z_t$ s.t. $\vert z_t - \hat{z}_t\vert = \vert \langle \hat{\alpha} - \alpha_0, x_t \rangle \leq \kappa_s$. Let $r_s =\frac{1}{2B_1 \sqrt{T_s}}$ Then for any $z\in [-W, 1+W]$, the probability that there exists a point $y\in \{\hat{z}_t\}_{t\in \Gamma_s}$ s.t that $|y -z|\leq r_s + \kappa_s$, is at least, 
\begin{eqnarray*}
1 - (1 - 2B_1\cdot r_s)^{T_s} = 1 - \left(1 - \frac{1}{\sqrt{T_s}}\right)^{T_s} \approx 1- e^{-\sqrt{T_s}} \geq 1 - \delta/2
\end{eqnarray*}
Therefore, for any $z\in [-W, 1+W]$, we can decompose $\vert\hat{F}_s(z; \hat{\alpha}_s) - F(z)\vert$ in the following,
\begin{eqnarray*}
&&\vert\hat{F}_s(z; \hat{\alpha}_s) - F(z)\vert\\
&\leq& \vert\hat{F}_s(z; \hat{\alpha}_s) - \hat{F}_s(y; \hat{\alpha}_s)\vert + \vert \hat{F}_s(y; \hat{\alpha}_s) - \hat{\mathbb{F}}_s(y)\vert + \vert \hat{\mathbb{F}}_s(y) - F(y)\vert + \vert F(y) - F(z)\vert
\end{eqnarray*}
Indeed, the characterization results by Lemma~\ref{lem:empirical-cdf-knots} applies to samples $\hat{z}_t$. Then we have $\vert \hat{F}_s(y; \hat{\alpha}_s) - \hat{\mathbb{F}}_s(y)\vert \leq \frac{1}{T_s}$. By the Lipshitzness of $\hat{F}_s(\cdot; \hat{\alpha}_s)$ and $F$, Eq.~(\ref{eq:empirical-cdf-convergence-episode}) and union bound, we have
\begin{eqnarray*}
\vert\hat{F}_s(z; \hat{\alpha}_s) - F(z)\vert &\leq& 2B_2 (\kappa_s + r_s) + \sqrt{\frac{\log(8/\delta)}{2T_s}} + B_2\kappa_s + \frac{1}{T_s}\\
&=& 3B_2\kappa_s + \frac{B_2}{B_1\sqrt{T_s}} + \sqrt{\frac{\log(8/\delta)}{2T_s}} + \frac{1}{T_s}
\end{eqnarray*}
holds with probability at least $1-\delta$ when $T_s \gg \log^2(2/\delta)$.

\end{proof}

With the help of the above lemmas, we show the full proof of Theorem~\ref{thm:full-info-regret} in the following,

\FullInfoRegret*

\begin{proof}[Proof of Theorem~\ref{thm:full-info-regret}]
We first rewrite the regret at each time step $t$ in episode $s$, in the following way,
\begin{eqnarray*}
\rgt_t &=& (\beta_0(x_t) - b^*_t) F(b^*_t - \alpha_0 \cdot x_t) - (\beta_0(x_t) - b_t) F(b_t - \alpha_0 \cdot x_t)\\
&=& (\beta_0(x_t) - b^*_t) F(b^*_t - \alpha_0 \cdot x_t) - (\beta_0(x_t) - b^*_t) F(b^*_t - \hat{\alpha}_{s-1} \cdot x_t)\\
&& + (\beta_0(x_t) - b^*_t) F(b^*_t - \hat{\alpha}_{s-1} \cdot x_t) - (\beta_0(x_t) - b^*_t) \hat{F}_{s-1}(b^*_t - \hat{\alpha}_{s-1} \cdot x_t; \hat{\alpha}_{s-1})\\
&& + (\beta_0(x_t) - b^*_t) \hat{F}_s(b^*_t - \hat{\alpha}_{s-1} \cdot x_t; \hat{\alpha}_{s-1}) - (\beta_0(x_t) - b_t) \hat{F}_{s-1}(b_t - \hat{\alpha}_{s-1} \cdot x_t; \hat{\alpha}_{s-1})\\
&& + (\beta_0(x_t) - b_t) \hat{F}_{s-1}(b_t - \hat{\alpha}_{s-1} \cdot x_t; \hat{\alpha}_{s-1}) - (\beta_0(x_t) - b_t) F(b_t - \hat{\alpha}_{s-1} \cdot x_t)\\
&& + (\beta_0(x_t) - b_t) F(b_t - \hat{\alpha}_{s-1} \cdot x_t) - (\beta_0(x_t) - b_t) F(b_t - \alpha_0 \cdot x_t)
\end{eqnarray*}

Note, by the definition of $b_t$, $(\beta_0(x_t) - b^*_t) \hat{F}_{s-1}(b^*_t - \hat{\alpha}_{s-1} \cdot x_t; \hat{\alpha}_{s-1}) - (\beta_0(x_t) - b_t) \hat{F}_{s-1}(b_t - \hat{\alpha}_{s-1} \cdot x_t; \hat{\alpha}_{s-1}) \leq 0$.
Given $F$ is $B_2$-Lipschitz on $[-W, 1+W]$, $\vert\beta_0(x_t) - b_t\vert\leq 1$ and $\vert\beta_0(x_t) - b^*_t\vert\leq 1$, we can bound the regret at time $t\in \Gamma_s$ as follows,
\begin{eqnarray*}
\rgt_t &\leq& 2B_2\vert x_t \cdot (\hat{\alpha}_{s-1} - \alpha_0)\vert + \left|\hat{F}_{s-1}(b^*_t - \hat{\alpha}_{s-1}\cdot x_t; \hat{\alpha}_{s-1}) - F(b_t^* - \hat{\alpha}_{s-1}\cdot x_t)\right|\\
& & + \left|\hat{F}_{s-1}(b_t - \hat{\alpha}_{s-1}\cdot x_t; \hat{\alpha}_{s-1}) - F(b_t - \hat{\alpha}_{s-1}\cdot x_t)\right|\\
&\leq& 2B_2\Vert \hat{\alpha}_{s-1} - \alpha_0\Vert_2 +  \left|\hat{F}_{s-1}(b^*_t - \hat{\alpha}_{s-1}\cdot x_t; \hat{\alpha}_{s-1}) - F(b_t^* - \hat{\alpha}_{s-1}\cdot x_t)\right|\\
& & + \left|\hat{F}_{s-1}(b_t - \hat{\alpha}_{s-1}\cdot x_t; \hat{\alpha}_{s-1}) - F(b_t - \hat{\alpha}_{s-1}\cdot x_t)\right|\\
\end{eqnarray*}

Let $\kappa_s = \frac{4A\sqrt{d}}{\tilde{\ell}_W\lambda_2}$ and $A$ is defined in Lemma~\ref{lem:full-info-convergence}, then the regret achieved in episode $s$ can be bounded,
\begin{eqnarray*}
\mathtt{Regret}_s &\leq& 2B_2T_s\Vert \hat{\alpha}_{s-1} - \alpha_0\Vert_2 + \sum_{t\in \Gamma_s} \left|\hat{F}_{s-1}(b^*_t - \hat{\alpha}_{s-1}\cdot x_t; \hat{\alpha}_{s-1}) - F(b_t^* - \hat{\alpha}_{s-1}\cdot x_t)\right|\\
&& + \sum_{t\in \Gamma_s}\left|\hat{F}_{s-1}(b_t - \hat{\alpha}_{s-1}\cdot x_t; \hat{\alpha}_{s-1}) - F(b_t - \hat{\alpha}_{s-1}\cdot x_t)\right|\\
&\leq & \frac{8B_2AT_s\sqrt{d}}{\tilde{\ell}_W\lambda_2} + 2T_s\left( 3B_2\frac{4AT_s\sqrt{d}}{\tilde{\ell}_W\lambda_2} + \frac{B_2}{B_1\sqrt{T_s}} + \sqrt{\frac{\log(16S/\delta)}{2T_s}} + \frac{1}{T_s}\right),
\end{eqnarray*}
where the second inequality holds when Lemma~\ref{lem:full-info-convergence} holds and Lemma~\ref{lem:empirical-cdf-convergence} (setting $\delta := \delta/2S$) holds simultaneously. Then the inequality holds with probability at least $1-\delta/S$ by union bound.

By union bound over $S$ episodes and the fact that $\frac{T_s}{T_{s-1}} = \sqrt{T}$ and $S \leq \log\log T$, we complete our proof.
\end{proof}

\section{Omitted Proofs from Section~\ref{sec:lower-bound}}\label{app:lower-bound}

\subsection{Proof of Theorem~\ref{thm:lower-bound-known-F}}

To prove this Theorem, we first propose the following auxiliary lemmas. 

\begin{lemma}\label{lem:lower-bound-1}
Suppose $b_t^* > 0$, there exists a constant $c_1 > 0$ (depending on $W$ and $\sigma$) such that $u''_t(b_t^*) < - c_1$ for any $t \geq 1$. Further, there exists constant $\delta > 0$ (depending on $W$ and $\sigma$) such that $u''_t(b) \leq -c_1/4$ for $b \in [b_t^* - \delta, b_t^* + \delta]$.
\end{lemma}

\begin{proof}
Let $g(\cdot)$ and $G(\cdot)$ be the PDF and CDF of the standard normal distribution, respectively. Then we can write the utility function and its derivatives in the following way,
\begin{eqnarray*}
u_t(b) &=& (\beta_0(x_t) - b)\cdot G\left(\frac{b - \alpha_0 \cdot x_t}{\sigma}\right)\\
u'_t(b) &=& - G\left(\frac{b - \alpha_0 \cdot x_t}{\sigma}\right) + \frac{1}{\sigma}(\beta_0(x_t) - b)\cdot g\left(\frac{b - \alpha_0 \cdot x_t}{\sigma}\right)\\
u''_t(b) &=& \frac{1}{\sigma}\left[\frac{b - \beta_0(x_t)}{\sigma} \left(\frac{b - \alpha_0 \cdot x_t}{\sigma}\right) - 2\right]\cdot g\left(\frac{b - \alpha_0 \cdot x_t}{\sigma}\right)
\end{eqnarray*}

Considering the optimal bid $b^*_t > 0$,
by the first-order conditions, we have
\begin{eqnarray}\label{eq:first-order-opt-bid}
\beta_0(x_t) - b_t^* = \frac{F(b_t^* - \alpha_0 \cdot x_t)}{f(b_t^* - \alpha_0 \cdot x_t)},
\end{eqnarray}
which implies $\frac{\beta_0(x_t) - b_t^*}{\sigma} = \frac{G((b_t^* - \alpha_0 \cdot x_t)/\sigma)}{G((b_t^* - \alpha_0 \cdot x_t)/\sigma)}$. Denote $\zeta = b_t^* - \alpha_0 \cdot x_t$, we have the following results for $u''_t(b_t^*)$
\begin{eqnarray*}
u''_t(b_t^*)  = \frac{1}{\sigma} \left[\frac{\zeta}{\sigma} \cdot \frac{-G(\zeta/\sigma)}{g(\zeta/\sigma)} - 2\right] \cdot g(\zeta/\sigma)
\end{eqnarray*}

Based on the concentration inequality (Feller, 1968), we have $G(x) \leq \frac{1}{-x}\frac{e^{-x^2/2}}{\sqrt{2\pi}}, \forall x < 0$. Then for any $\zeta < 0$, we have $\frac{\zeta}{\sigma} \cdot \frac{-G(\zeta/\sigma)}{g(\zeta/\sigma)} - 2 \leq -1$. For any $\zeta > 0$, $\frac{\zeta}{\sigma} \cdot \frac{-G(\zeta/\sigma)}{g(\zeta/\sigma)} - 2 \leq -2$ holds trivially.

By definition of the function $\vert \zeta \vert = \vert b_t^* - \alpha_0\cdot x_t\vert \leq 1 + W$. By the property of the normal distribution, $g(\zeta/\sigma) \geq g((1+W)/\sigma)$. Thus, we get $u''_t(b_t^*) \leq -\frac{1}{\sigma} g((1+W)/\sigma) := -c_1$, when $b_t^* > 0$.

Next, we consider any $b \in [b_t^* - \delta, b_t^* + \delta]$, where $\delta = \min\{\frac{\sigma^2}{2(W+3)}, \frac{g((1+W)/\sigma)\sigma^2}{2}\}$. Then we have
\begin{eqnarray*}
\left|\frac{b - \beta_0(x_t)}{\sigma} \left(\frac{b - \alpha_0 \cdot x_t}{\sigma}\right) - \frac{b^*_t - \beta_0(x_t)}{\sigma} \left(\frac{b_t^* - \alpha_0 \cdot x_t}{\sigma}\right)\right| &\leq& \frac{1}{\sigma^2} \cdot \vert b - b_t^*\vert \cdot \vert b + b_t^* - (\beta_0(x_t) + \alpha_0\cdot x_t)\vert\\
&\leq& \frac{\delta (W + 3)}{\sigma^2} \leq \frac{1}{2}
\end{eqnarray*}

In addition, we have
\begin{eqnarray}
\left\vert g\left(\frac{b - \alpha_0 \cdot x_t}{\sigma}\right) - g\left(\frac{b_t^* - \alpha_0 \cdot x_t}{\sigma}\right)\right\vert \leq \frac{|b - b_t^*|}{2\sigma^2} \leq \frac{\delta}{\sigma^2} \leq \frac{1}{2}g((1+W)/\sigma),
\end{eqnarray}
which implies $g\left(\frac{b - \alpha_0 \cdot x_t}{\sigma}\right) \geq g(\zeta/\sigma) - \frac{1}{2}g((1+W)/\sigma) \geq \frac{1}{2}g((1+W)/\sigma)$.

As discussed above, $\frac{b^*_t - \beta_0(x_t)}{\sigma} \left(\frac{b_t^* - \alpha_0 \cdot x_t}{\sigma}\right) \leq 1$. Therefore, for any $b\in [b_t^* - \delta, b_t^* + \delta]$,
\begin{eqnarray*}
u''_t(b) &\leq& \frac{1}{\sigma}\cdot g\left(\frac{b - \alpha_0 \cdot x_t}{\sigma}\right) \cdot \left[\frac{b - \beta_0(x_t)}{\sigma} \left(\frac{b - \alpha_0 \cdot x_t}{\sigma}\right) - 2\right]\\
&\leq& \frac{1}{\sigma}\cdot g\left(\frac{b - \alpha_0 \cdot x_t}{\sigma}\right) \cdot \left[\frac{b^*_t - \beta_0(x_t)}{\sigma} \left(\frac{b_t^* - \alpha_0 \cdot x_t}{\sigma}\right) + \frac{1}{2} - 2\right] \\
&\leq & -\frac{1}{4\sigma} g((1+W)/\sigma)  = -\frac{1}{4}c_1
\end{eqnarray*}
\end{proof}

\begin{lemma}[\citet{JN19}]\label{lem:lower-bound-2}
Let $x\in \R^d$ be a random vector such that its coordinates are chosen independently and uniformly at random from $\{-1, 1\}$. Further, suppose that $v\in \R^d$ and $\delta > 0$ are deterministic. Then,
\begin{eqnarray*}
\E\left[\min((x\cdot v)^2, \delta^2)\right] \geq 0.1 \min(\Vert v\Vert_2^2, \delta^2).
\end{eqnarray*}
\end{lemma}

\begin{lemma}[\citet{JN19}]\label{lem:lower-bound-3}
Consider the linear model~(\ref{eq:linear-model-competitor}) and assume that the maximum bids from competitors $m_t, 1\leq t \leq T$, are fully observed and the context $x_t$ is i.i.d generated such that its coordinates are chosen independently and uniformly at random from $\{-1, 1\}$ at each time $t$.
We further assume that the noise in market value is generated as $z_t\sim \mathcal{N}(0, \sigma^2)$. Then, conditional on historical contexts $(x_1,\cdots, x_T)$, and
for any fixed value $C > 0$, there exists a nonnegative constant $\tilde{C}$, depending on $C, \sigma, W$, such that
\begin{eqnarray*}
\min_{\alpha_1,\alpha_2,\cdots, \alpha_T}\max_{\alpha_0: \Vert\alpha_0\Vert_1\leq W} \E\left[\min\{\Vert \alpha_t - \alpha_0\Vert_2^2, C\}\right] \geq \tilde{C}\sqrt{T \log T}
\end{eqnarray*}
\end{lemma}

\begin{proof}[Proof of Theorem~\ref{thm:lower-bound-known-F}]
Firstly, we represent the regret at time $t$ as $\rgt_t = u_t(b_t^*) - u_t(b_t)$. Setting $\beta_0(x_t) \geq  \delta + \frac{\sigma}{g((\delta+W)/\sigma)}$ (it is without loss generality to assume $\delta + \frac{\sigma}{g((\delta+W)/\sigma)} \leq 1$), for any $x_t$. 

Then we prove $b_t^* \geq \delta$. Indeed, we have
\begin{eqnarray*}
\beta_0(x_t) \geq \delta + \frac{\sigma G((\delta - \alpha_0\cdot x_t)/\sigma)}{g((\delta+W)/\sigma)} \geq \delta + \frac{\sigma G((\delta - \alpha_0\cdot x_t)/\sigma)}{g((\delta- \alpha_0\cdot x_t)/\sigma)} = \delta + \frac{F(\delta - \alpha_0\cdot x_t)}{g(\delta - \alpha_0 \cdot x_t)},
\end{eqnarray*}
where the first inequality holds because $G((\delta - \alpha_0\cdot x_t)/\sigma)\leq 1$ and the second inequality holds based on the property of standard normal distribution and $\alpha_0 \cdot x_t \leq W$. Given the above inequality, we have
\begin{eqnarray*}
\alpha_0 \cdot x_t  + \varphi^{-1}(\beta_0(x_t) - \alpha_0 \cdot x_t) \geq \alpha_0 \cdot x_t + \varphi^{-1}\left(\delta- \alpha_0 \cdot x_t + \frac{F(\delta - \alpha_0\cdot x_t)}{g(\delta - \alpha_0 \cdot x_t)}\right) = \delta\\
\end{eqnarray*}

Therefore $b_t^* \geq \delta$, by the second order Taylor's theorem,
\begin{eqnarray*}
\rgt_t = -\frac{1}{2}u''(\tilde{b})(b_t - b_t^*)^2,
\end{eqnarray*} 
for some $\tilde{b}$ between $b_t$ and $b_t^*$. For any $b_t$ generated from a bidding policy in $\Pi$, 
\begin{itemize}
\item When $b_t > 0$, we have
\begin{eqnarray*}
(b_t - b_t^*)^2  &=& \left(\alpha_t\cdot x_t + \varphi^{-1}(\beta_0(x_t) - \alpha_0\cdot x_t) - \alpha_0\cdot x_t - \varphi^{-1}(\beta_0(x_t) - \alpha_t\cdot x_t)\right)^2\\
&\geq & c_2^2\vert x_t \cdot (\alpha_t - \alpha_0)\vert^2
\end{eqnarray*}
for some constant $c_2 \geq 0$, since $(x + \varphi^{-1}(v - x))' \geq c_2$ over bounded interval $x\in [- W, v]$ and $v \leq [0, 1]$.
\item When $b_t = 0$, we have
\begin{eqnarray*}
(b_t - b_t^*)^2 \geq \delta^2
\end{eqnarray*}
\end{itemize}

We consider the following two cases,
\begin{itemize}
\item $|b_t - b_t^*| \leq \delta$, then we have $\tilde{b} \in [b_t^* - \delta, b_t^* + \delta]$. By Lemma~\ref{lem:lower-bound-1}, we obtain,
\begin{eqnarray*}
\rgt_t = -\frac{1}{2}u''_t(\tilde{b})(b_t - b_t^*)^2 \geq \frac{c_1}{8}(b_t - b_t^*)^2
\end{eqnarray*}

\item $|b_t - b_t^*| > \delta$, since $u_t(\cdot)$ has only local maximum, $b_t^*$, the function is increasing before $b_t^*$ and decreasing afterwards, then for any $b_t \leq b_t^* -\delta$ and context $x_t$,
\begin{eqnarray*}
u_t(b_t) \leq u_t(b_t^* - \delta) =  u_t(b_t^*) + \frac{1}{2}u''_t(b)\delta^2 \leq u_t(b_t^*) - \frac{c_1}{8}\delta^2,
\end{eqnarray*}
for some $b \in [b_t^* - \delta, b_t^*]$. Similar result holds for $b_t \geq b_t^* +\delta$.
\end{itemize}

Thus, we have the lower bound for the regret at each time $t$, such that
\begin{eqnarray*}
\rgt_t \geq \frac{c_1}{8} \E\left[\min\{(b_t - b_t^*)^2, \delta^2\}\right] \geq \frac{c_1}{8} \E\left[\min\left\{c_2^2\vert x_t \cdot (\alpha_t - \alpha_0)\vert^2, \delta^2\right\}\right]
\end{eqnarray*}

By Lemma~\ref{lem:lower-bound-2}, we have
\begin{eqnarray}
\rgt_t \geq \frac{c_1 c_2^2}{80} \E\left[\min\left\{\vert(\alpha_t - \alpha_0)\vert^2, \delta^2/c_2^2\right\}\right]
\end{eqnarray}

Then we can lower bound the min-max regret of any policy in $\Pi$.
\begin{eqnarray*}
R(T) \geq \max_{\alpha_0: \Vert\alpha_0\Vert_1\leq W} \sum_{t=1}^T rgt_t \geq  \frac{c_1 c_2^2}{80} \max_{\alpha_0: \Vert\alpha_0\Vert_1\leq W}\sum_{t=1}^T\E\left[\min\left\{\vert(\alpha_t - \alpha_0)\vert^2, \delta^2/c_2^2\right\}\right] \geq \Omega\left(\sqrt{T\log(T)}\right)
\end{eqnarray*}

\end{proof}

\section{Omitted Algorithms}\label{app:omitted-algorithms}

In this section, we provide the pseudo codes of bidding algorithms omitted in Section~\ref{sec:binary-partial-known-noise} and Section~\ref{sec:full-unknown-noise}. Algorithm~\ref{alg:binary-partial-known-noise-mle} is designed for the binary feedback model and the learner only partially knows the noise distribution, and Algorithm~\ref{alg:full-unknown-noise-mle} is for the full information feedback model.

\begin{algorithm}[h!]
\SetAlgoNoLine
\KwIn{Parameters $W, \Delta, T, T_1, T_2,\cdots, T_S$, function $\varphi_0^{-1}(\cdot)$.}
\For{$t\in \Gamma_1$}{
The learner observes $x_t$ and submits a bid $b_t = 1$. 

The learner observes $\delta_t$.
}

Estimate $\mu_0, \rho_0$ by using $\hat{\mu}_1, \hat{\rho}_1$, which is computed by,
{
\begin{eqnarray}\label{eq:binary-partial-known-noise-opt-mle-init}
\hat{\mu}_1, \hat{\rho}_1 = \argmin_{(\mu, \rho)\in \Lambda} \mathcal{L}_1(\mu, \rho).
\end{eqnarray}
}
\For{episode $s=2,3,\cdots, S$}{
\For{$t\in \Gamma_s$}{
The learner observes $x_t$ and submits $b_t$, where $b_t$ is computed in the following way,
\begin{eqnarray}\label{eq:bid-binary-partial-known-mle}
b_t = \max\left\{\Delta, \frac{1}{\hat{\rho}_{s-1}} \left(\hat{\mu}_{s-1}\cdot x_t + \varphi^{-1}_0(\hat{\rho}_{s-1}\beta_0(x_t) - \hat{\mu}_{s-1}\cdot x_t)\right)\right\}
\end{eqnarray}
The learner observes $\delta_t$.
}

Update the estimator for $\mu_0, \rho_0$ in the episode $s$ by using $\hat{\mu}_s, \hat{\rho}_0$, which is computed as below,

\begin{eqnarray}\label{eq:binary-partial-known-noise-opt-mle}
\hat{\mu}_s, \hat{\rho}_s = \argmin_{(\mu, \rho)\in \Lambda} \mathcal{L}_s(\mu, \rho),
\end{eqnarray}
}

\caption{Bidding algorithm in the binary feedback model with partially-known noise distribution}\label{alg:binary-partial-known-noise-mle}
\vspace{-2pt}
\end{algorithm}

\begin{algorithm}[h!]
\SetAlgoNoLine
\KwIn{Parameters $W, T, \kappa_s$, function $\varphi^{-1}(\cdot)$}
\For{$t\in \Gamma_1$}{
The learner observes $x_t$ and submits a bid $b_t = 1$. 

The learner observes $m_t$.
}

Estimate $\alpha_0$ by using $\hat{\alpha}_1$, which is computed by,
\begin{eqnarray}
\hat{\alpha}_1 = \argmin_{\Vert\alpha\Vert_1 \leq W} \mathcal{L}_1(\alpha),
\end{eqnarray}
where $\L_1(\alpha)$ is defined in Eq.~(\ref{eq:full-info-log-likelihood}).

Compute $\hat{\Phi}(\cdot; \hat{\alpha}_1)$ s.t.
\begin{eqnarray*}
\hat{\Psi}_1(\cdot; \hat{\alpha}_1) = \argmax_{\Psi \mbox{ is concave}, \Psi\leq \log B_2} \frac{1}{T_1} \sum_{t\in \Gamma_1} \Psi (m_t - \hat{\alpha}_1\cdot x_t) - \int \exp\left(\Psi(z; \hat{\alpha}_1)\right) dz
\end{eqnarray*}

Compute estimation of $F$ as
$$\hat{F}_1(z; \hat{\alpha}_1) = \int^z \exp(\hat{\Psi}_1(\cdot; \hat{\alpha}_1))dz$$

\For{episode $s=2,3,\cdots, S$}{
\For{$t\in \Gamma_s$}{
The learner observes $x_t$ and submits $b_t$, where $b_t$ is computed in the following way,
\begin{eqnarray*}
b_t = \argmax_{b\in [0, \beta_0(x_t)]} (\beta_0(x_t) - b)\hat{F}_{s-1}(b - \hat{\alpha}_{s-1}\cdot x_t; \hat{\alpha}_{s-1})
\end{eqnarray*}

The learner observes $m_t$.
}

Update the estimator for $\alpha_0$ in episode $s$ by using $\hat{\alpha}_s$, which is computed as below,
\begin{eqnarray}
\hat{\alpha}_s = \argmin_{\Vert\alpha\Vert_1 \leq W} \mathcal{L}_s(\alpha),
\end{eqnarray}
where $\L_s(\alpha)$ is defined in Eq.~(\ref{eq:full-info-log-likelihood}).

Compute $\hat{\Phi}(\cdot; \hat{\alpha}_1)$ s.t.
\begin{eqnarray*}
\hat{\Psi}_s(\cdot; \hat{\alpha}_s) = \argmax_{\Psi \mbox{ is concave}, \Psi\leq \log B_2} \frac{1}{T_s} \sum_{t\in \Gamma_s} \Psi (m_t - \hat{\alpha}_s\cdot x_t) - \int \exp\left(\Psi(z; \hat{\alpha}_s)\right) dz
\end{eqnarray*}

Update estimation of $F$ as 
$$\hat{F}_s(z; \hat{\alpha}_s) =  \int^z \exp(\hat{\Psi}_s(\cdot; \hat{\alpha}_s))dz$$
}

\caption{Bidding algorithm in the full information feedback model with unknown noise distribution}\label{alg:full-unknown-noise-mle}
\vspace{-2pt}
\end{algorithm}

\section{Omitted Discussions}\label{app:omitted-discussion}

\subsection{Discussion on Informational Bids}
In the setting that the noise distribution $F$ is known and satisfies Assumption~\ref{assump:log-concave} and Assumption~\ref{assump:bounded-noise}, there is no "uninformational" bids. This follows the same argument in Section 4.1 in~\citep{JN19}. In fact, for any bid $b$ and any parameters $\alpha_1, \alpha_2$, denote $d_t(b, \theta) = F(b - \alpha \cdot x_t)$ be the winning probability given parameter $\alpha$. Let $d_1^n(b, \alpha) = (d_1(b, \alpha), \cdots, d_n(b, \alpha))$, then we have
\begin{eqnarray*}
\Vert d^n_1(b, \alpha_1) - d^n_1(b, \alpha_2)\Vert_2^2 &=&\sum_{t=1}^n \left(F(b - \alpha_1 \cdot x_t) - F(b - \alpha_2 \cdot x_t)\right)^2\\
&\geq& \sum_{t=1}^n \left((\alpha_1 - \alpha_2) \cdot x_t\right)^2\\
&\geq & B_1\Vert X(\alpha_1 - \alpha_2)\Vert_2^2\\
&\geq& \frac{\lambda_2}{2}\Vert \alpha_1 - \alpha_2\Vert_2^2
\end{eqnarray*}
where $X$ is the matrix with rows $x_t, 1\leq t\leq n$, the second last inequality holds because $f(z)\geq B_1, \forall z\in [-W, 1+W]$, and the last inequality holds with high probability based on the generation of $x_t$ in the statement of Theorem~\ref{thm:lower-bound-known-F}.

Therefore, for any bid $b$, if we vary $\alpha_1$ to $\alpha_2$, the aggregated winning probability at bid $b$ also changes by an amount $\Theta(\Vert \alpha_1 - \alpha_2\Vert_2)$. Thus any bid in this setting (defined in Theorem~\ref{thm:lower-bound-known-F}) is informative.

\end{document}